\documentclass[9pt,twocolumn,twoside]{./style_files/ilcss}
\templatetype{./style_files/ilcssworkingpaper} % Choose template 
\usepackage{lipsum}  
\usepackage{amsmath,amsfonts,amssymb,amsthm,array}
\usepackage{algorithm}
\usepackage{caption}
\usepackage{subcaption}
\usepackage{algpseudocode}
\usepackage{bm}
\usepackage{color}
\usepackage{graphicx}
\usepackage{tcolorbox}
\usepackage{wrapfig}
\usepackage{xspace}
\usepackage{picinpar}
\graphicspath{{./figures/}}

\newcommand{\E}[1]{{\mathbf{E}\left[#1\right] }}    % expectation
\newcommand{\Prob}[1]{\mathbf{P} \left[ #1\right]}
\newcommand{\EE}[2]{{\mathbf{E}_{#1}\left[#2\right] }}  % Conditional expectation

\newcommand{\ADMM}{{\tt ADMM}\xspace}
\newcommand{\GD}{{\tt GD}\xspace}
\newcommand{\SGD}{{\tt SGD}\xspace}
\newcommand{\SGDstar}{{\tt SGD$_\star$}\xspace}
\newcommand{\IAG}{{\tt IAG}\xspace}
\newcommand{\SAG}{{\tt SAG}\xspace}
\newcommand{\SAGA}{{\tt SAGA}\xspace}
\newcommand{\SVRG}{{\tt SVRG}\xspace}
\newcommand{\StwoGD}{{\tt S2GD}\xspace}
\newcommand{\SDCA}{{\tt SDCA}\xspace}
\newcommand{\ASDCA}{{\tt ASDCA}\xspace}
\newcommand{\MISO}{{\tt MISO}\xspace}
\newcommand{\VR}{{\tt VR}\xspace}
\newcommand{\SQN}{{\tt SQN}\xspace}
\newcommand{\BFGS}{{\tt BFGS}\xspace}
\newcommand{\LBFGS}{{\tt L-BFGS}\xspace}

\newcommand{\R}{\mathbb{R}}
\newcommand{\N}{\mathbb{N}}

\providecommand{\norm}[1]{{\left\| #1\right\|}}
\newcommand{\dotprod}[1]{\left< #1\right>}

\DeclareMathOperator*{\argmin}{\arg\!\min}
\DeclareMathOperator*{\argmax}{\arg\!\max}
\usepackage{tikz}

\newcommand{\eqdef}{:=}

\newcommand{\cL}{{\cal L}}

\newcommand{\cO}{{\cal O}}

\newcommand{\mA}{{\bf A}}

\newcommand{\mH}{H}

\usepackage{mdframed} 
\usepackage{thmtools}

\definecolor{shadecolor}{gray}{0.9}
\declaretheoremstyle[
headfont=\normalfont\bfseries,
notefont=\mdseries, notebraces={(}{)},
bodyfont=\normalfont,
postheadspace=0.5em,
spaceabove=1pt,
mdframed={
  skipabove=8pt,
  skipbelow=8pt,
  hidealllines=true,
  backgroundcolor={shadecolor},
  innerleftmargin=4pt,
  innerrightmargin=4pt}
]{shaded}

\declaretheorem[within=section]{definition}
\declaretheorem[sibling=definition]{theorem}

\declaretheorem[sibling=definition]{assumption}

\declaretheorem[sibling=definition]{lemma}

\usepackage[colorinlistoftodos,bordercolor=orange,backgroundcolor=orange!20,linecolor=orange,textsize=scriptsize]{todonotes}

\newcommand{\revised}[1]{ #1}
\newcommand{\schmidt}[1]{}
\newcommand{\peter}[1]{}
\newcommand{\rob}[1]{}
\newcommand{\francis}[1]{}
\usepackage{hyperref}% same for example numbers

\title{Variance-Reduced Methods for  Machine Learning }

\author[a]{Robert M. Gower}
\author[b]{Mark Schmidt}
\author[c]{Francis Bach} 
\author[d]{Peter Richt\'arik}

\affil[a]{LTCI, T\'{e}l\'{e}com Paris, Institut Polytechnique de Paris }
\affil[c]{Inria - PSL Research University, France}
\affil[b]{University of British Columbia, CCAI Affiliate Chair (Amii), Canada}
\affil[d]{King Abdullah University of Science and Technology, Kingdom of Saudi Arabia}

\correspondingauthor{\textsuperscript{2}To whom correspondence should be addressed. E-mail: gower.robert@gmail.com}


\begin{abstract}
Stochastic optimization lies at the heart of machine learning, and its cornerstone is stochastic gradient descent (\SGD), a method introduced over 60 years ago. The last 8 years have seen an exciting new development: variance reduction (\VR) for stochastic optimization methods. These \VR methods excel in settings where more than one pass through the training data is allowed, achieving a faster convergence than \SGD in theory as well as practice. These speedups underline the surge of interest in \VR methods and the fast-growing body of work on this topic. This review covers the key principles and main developments behind \VR methods for optimization with finite data sets and is aimed at non-expert readers. 
 \revised{We focus mainly on the convex setting, and leave pointers to readers interested in extensions for minimizing non-convex functions.   }
\end{abstract}

\begin{document}

\maketitle
\thispagestyle{firststyle}
\ifthenelse{\boolean{shortarticle}}{\ifthenelse{\boolean{singlecolumn}}{\abscontentformatted}{\abscontent}}{ssss}
% \fi
%\IEEEpeerreviewmaketitle
%

%Instructions on preparing the paper:\\
%\url{http://proceedingsoftheieee.ieee.org/instructions-for-authors/preparing-your-special-issue-paper/}
%
%Important:
%\begin{enumerate}
%\item Only 12 pages
%\item Should be accessible to non-experts
%\item Authors {\bf can} use the standard Transactions template for Proceedings of the IEEE. 
%\end{enumerate}
%
%Questions:
%\begin{enumerate}
%\item Should we present aW convergence theorem for each method? Or maybe simplify and say which ones achieve $O\left( \left(n+\frac{L_{\max}}{\mu} \right)\log\left( \frac{1}{\varepsilon} \right)\right)$ iteration complexity.
%\end{enumerate}
%\tableofcontents

%\rob{Missing references: Co-coa~\cite{cocoa}, early variance reduction reference~\cite{WangSmola2013}, Miso~\cite{MISO} and Finito~\cite{Finito}}
%Consider the optimization problem

\section{Introduction}

One of the fundamental problems studied in the field of machine learning is how to fit models to large datasets. For example, consider the classic linear least squares model,
\begin{equation}
    \label{eq:leastSquares}
 x_\star   \in \argmin_{x \in \R^d} \left\{\frac{1}{n} \sum_{i=1}^{n} (a_i^\top x  - b_i)^2 \right\}.
\end{equation}
Here, the model has $d$ parameters given by the vector $x \in \R^d$ and we are given $n$ data points $\{a_i, b_i\}$ consisting of feature vectors $a_i \in \R^d$ and target values (labels) $b_i \in \R$. Fitting the model consists of tuning these $d$ parameters so that the model's output $a_i^\top x$ is ``close'' (on average) to the  targets $b_i$. More generally, we might use some \emph{loss function} $f_i(x)$ to measure how close our model is to the $i$-th data point,
\begin{equation}
    \label{eq:optimization_problem}
 x_\star   \in \argmin_{x \in \R^d} \left\{ f(x) \eqdef \frac{1}{n} \sum_{i=1}^{n} f_i (x) \right\}.
\end{equation}
 If $f_i(x)$ is large, we say that our model's output is far from the data, and if $f_i(x) = 0$ we say that our model fits perfectly the $i$-th data point. The function $f(x)$ represents the average \emph{loss} of our model over the full dataset. A problem of the form~\eqref{eq:optimization_problem} characterizes the training of not only linear least squares, but many models studied in machine learning. For example, the logistic regression model solves
 \begin{equation}
    \label{eq:logistic}
 x_\star   \in \argmin_{x \in \R^d} \left\{\frac{1}{n} \sum_{i=1}^{n} \log(1+\exp(-b_ia_i^\top x)) + \frac{\lambda}{2}\norm{x}^2\right\},
\end{equation}
 where we are now considering a binary classification task with $b_i \in \{-1,+1\}$ (and predictions are made using the sign of $a_i^\top x$). Here, we have also used $ \frac{\lambda}{2}\norm{ x}^2 \eqdef \frac{\lambda}{2} \sum_{i=1}^d x_i^2$,  as a regularizer. This and other regularizers are commonly added to avoid overfitting to the given data, and in this case we replace each $f_i(x)$ by $f_i(x) + \frac{\lambda}{2} \norm{ x }^2$. The training procedure in most supervised machine learning models can be written in the form~\eqref{eq:optimization_problem}, including L1-regularized least squares, support vector machines, principal component analysis, conditional random fields, and deep neural networks.
  
 %, and use the gradient techniques that we describe directly. In general cases, such as when $\Omega(x)$ is proportional to the $\ell_1$-norm $\| x\|_1$, so-called ``proximal techniques'' can be used as little extra cost (see Section~\ref{sec:proximal}).
 
%  It is an average of $n\in \N$ individual $f_i$ functions where each $f_i(x)$ encodes the goodness of fit of our model with respect to the $i$-th data point. As such, we assume we have

A key challenge in modern instances of problem~\eqref{eq:optimization_problem} is that the number of data points $n$ can be extremely large. We regularly collect datasets going beyond terabytes, from sources such as the internet, satellites, remote sensors, financial markets, and scientific experiments. One of the most common ways to cope with such large datasets is to use \emph{stochastic gradient descent} (\SGD) methods, which use a few randomly chosen data points in each of their iterations. Further, there has been a recent surge in interest in variance-reduced (\VR) stochastic gradient methods which converge faster than classic stochastic gradient methods.

%Optimization problems such as~\eqref{eq:optimization_problem} date back to at least the 50's~\citep{RobbinsMonro:1951}. 
%What has changed since finite sum problems such as~\eqref{eq:optimization_problem} first appeared is that now data is being collected automatically at an unprecedented rate from text and images streaming in from the internet, from satellites, remote sensors, and financial markets.  The resulting data sets can go beyond terabytes and are difficult to store in cache memory of most personal computers. To cope with these ever increasing data sets, new optimization methods have been developed that scale gracefully with the dimensions of these data. Here we delve into stochastic gradient methods that do just that by exploiting the finite sum structure in~\eqref{eq:optimization_problem}. 

\newtcolorbox{mybox}{colback=blue!5!white,colframe=blue!50!black}
\begin{mybox}
Stochastic variance-reduced methods are as cheap to update as \SGD, and also have a fast exponential convergence like full gradient descent.
\begin{figure}[H] \centering
\includegraphics[scale=0.45]{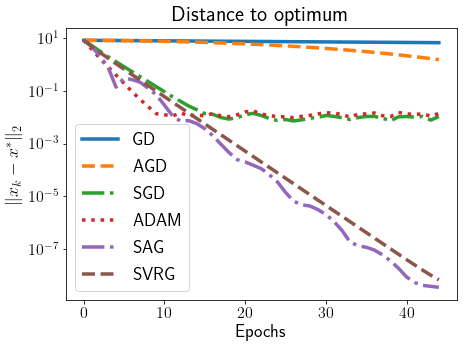} %opt1 opt2 all_solvers_mushrooms
\caption{Comparison of the \GD, \texttt{AGD} (Accelerated \GD~\citep{nesterov1983method}), \SGD  and \texttt{ADAM}~\citep{ADAM} methods to the 
 \VR methods \SAG and \SVRG on a logistic regression problem based on the \texttt{mushrooms} data set~\citep{Chang2011}, where $n =  8,124$ and $d=112$.}
\label{fig:significant}
\end{figure}
\end{mybox}

%The first key insight into developing efficient methods for solving~\eqref{eq:optimization_problem}, is that for a method to be able to scale with the number of data points $n$, it should use only a few data points at each iteration. In other words, given only a hand full of the $f_i$ functions in the sum and their derivatives, our  method should be able to make progress towards the solution $x_\star$. As a counter example, take the \GD (gradient descent) method
\subsection{Gradient and Stochastic Gradient Descent}
The classic \GD (gradient descent) method applied to problem~\eqref{eq:optimization_problem} takes the form
\begin{equation}\label{eq:gradintro}
x_{k+1} = x_k - \gamma \frac{1}{n} \sum_{i=1}^n\nabla f_i(x_k),
\end{equation}
where $\gamma >0$ is a fixed stepsize\footnote{The classic way to implement \GD is to determine $\gamma$ as the approximate solution to 
$\min_{\gamma>0} f\left(x_k -\gamma \nabla f(x_k)\right)$. This is called a \emph{line search} since it is an optimization over a line segment~\citep{armijo1966,MM2019}. This line search requires multiple evaluations of the full objective function $f(x_k),$ which in our setting is too expensive since this would require loading all the data points multiples times. This is why we use fixed constant stepsize instead.}. At each iteration the \GD method needs to calculate a gradient $ \nabla f_i(x_k)$ for every $i$th data point, and thus %the entire data set must be loaded onto memory. In other words, 
 \GD takes a \emph{full pass over the $n$ data points} at each iteration. This expensive cost per iteration makes \GD prohibitive when $n$ is large.

%\begin{remark}[Passes over the data and Epochs] We will use \emph{Epoch} to mean an \emph{effective passes over the data}, that is, 
%\end{remark}

Consider instead the \emph{stochastic gradient descent} (\SGD) method,
\begin{equation}\label{eq:sgd}
x_{k+1} = x_k - \gamma  \nabla f_{i_k}(x_k),
\end{equation}
first introduced by~\cite{RobbinsMonro:1951}. It avoids the heavy cost per iteration of \GD by using one randomly-selected $\nabla f_{i_k}(x_k)$ gradient  instead of the full gradient. In Figure~\ref{fig:significant}, we see how the \SGD method makes dramatically more progress than \GD (and even the ``accelerated'' \GD method) in the initial phase of optimization. \revised{Note that this figure plots the progress in terms of the number of \emph{epochs}, which is the number of times we have computed $n$ gradients of individual training examples. The \GD method does one iteration per epoch while the \SGD method does $n$ iterations per epoch. We compare \SGD and \GD in terms of epochs taken since we assume that $n$ is very large and that the main cost of both methods is computing the  $\nabla f_i(x_k)$ gradients.}

%A key idea for avoiding this heavy cost per iteration is to use a stochastic unbiased estimate of the gradient instead of the full gradient. For instance, let $i_k \in \{1,\ldots, n\}$ be a random index such that $\Prob{i_k=i} = \frac{1}{n}$ for $i,\ldots, n.$ For short-hand we use $i_k \sim \frac{1}{n}.$ 
%\francis{Weird notation. Is this ever used later? If not, I would suggest to remove.}
%\rob{yep, about 5 times. We could replace it with the more explicit notation.}
%We now have that $\nabla f_{i_k}(x_k)$ is an unbiased estimate of $\nabla f(x_k)$ since 
%Using the \emph{stochastic gradient}  $\nabla f_{i_k}(x_k)$ instead of the full gradient in~\eqref{eq:gradintro} gives

%This method is known as the \SGD (stochastic gradient descent\footnote{Note that as illustrated in~\ref{fig:sgdconst}, \SGD is not always going down at every iteration; it does so on average.}) method and was first introduced by~\cite{RobbinsMonro:1951} over 60 years ago. Note that it applies more generally to any function of the form $f(x) = \mathbb{E}_z f(x,z)$, where $z$ is any random variable, potentially taking more than $n$ values. Thus, when only a single pass is made and the data is assumed to be sampled independently from the same distribution, \SGD leads to guarantees directly on the generalization performance (which is not computable given only observations). In this paper, generalization performance will be guaranteed through regularization. \francis{I have made an attempt at taling about test error.}

\subsection{The Issue with Variance}

Observe that if we choose the random index $i_k \in \{1,\ldots, n\}$ uniformly, $\Prob{i_k=i} = \frac{1}{n}$ for all~$i$,  then $\nabla f_{i_k}(x_k)$ is an unbiased estimate of $\nabla f(x_k)$ since 
\begin{align} \label{eq:sgdunbiased}
\E{\nabla f_{i_k} (x_k) \;|\; x_k} &= \sum_{i=1}^n \frac{1}{n}  \nabla  f_i (x_k) = \nabla f(x_k).
\end{align}
Thus, even though the \SGD method is not guaranteed to  decrease $f$ in each iteration, on average the method is moving in the direction of the negative full gradient, which is a direction of descent.

Unfortunately, having an unbiased estimator of the gradient is not enough to guarantee convergence of the iterates~\eqref{eq:sgd} of  \SGD. To illustrate this, in Figure~\ref{fig:sgdconst} (left) we have plotted the iterates of \SGD with a constant stepsize applied to a logistic regression function using the \texttt{fourclass} data set from LIBSVM~\citep{Chang2011}. The concentric ellipses in Figure~\ref{fig:sgdconst} are the \emph{level sets} of this function, that is, the points $x$ on a single ellipse are given by $\{x \,: \, f(x) =c\}$ for a particular constant $c \in \R.$ Different constants $c$ give different ellipses.

The iterates of \SGD do not converge to the solution (the green star), and instead form a point cloud around the solution. In contrast, we have plotted the iterates of a \VR method \SAG (that we present later) in Figure~\ref{fig:sgdconst} using the same constant stepsize.

\begin{figure}
\centering \vspace{-0.3cm}
\includegraphics[scale=0.235]{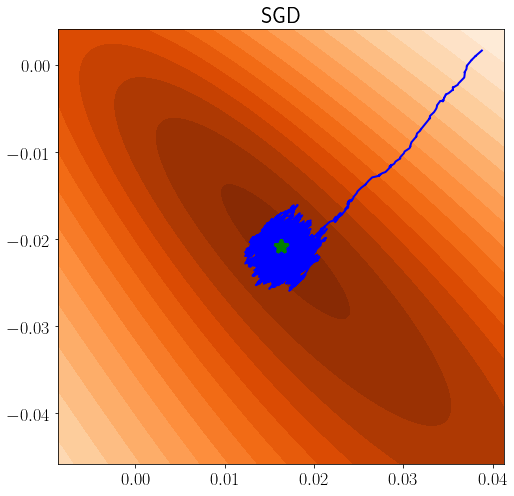}
\includegraphics[scale=0.235]{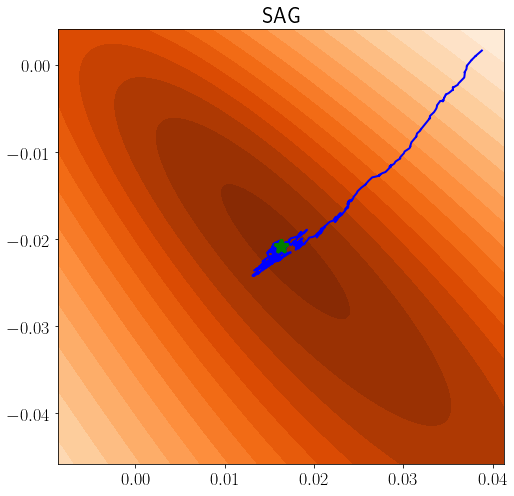}
\caption{Level set plot of 2D logistic regression with the iterates of \SGD (left) and \SAG (right) with constant stepsize. The green star is the $x_*$ solution.}\vspace{-0.3cm}
\label{fig:sgdconst}
\end{figure}
The reason why \SGD does not converge in this example is because the stochastic gradients themselves do no converge to zero, and thus the method~\eqref{eq:sgd} with a constant stepsize \emph{never stops}. This is in contrast with \GD, where the method naturally stops since $\nabla f(x_k) \rightarrow 0$ as $x_k \rightarrow x_\star$.%, which shows that the average of the  $\nabla f_{i} (x_\star)$ gradients is zero. But each individual $\nabla f_{i} (x_\star)$ can be different from zero.

\subsection{Classic Variance Reduction Methods}

There are several classic techniques for dealing with the non-convergence due to the variance in the $\nabla f_{i}(x_k)$ values. For example,~\cite{RobbinsMonro:1951} address the issue of the variance using a sequence of decreasing stepsizes $\gamma_k$. This forces the product $\gamma_k  \nabla f_{i_k}(x_k)$ to converge to zero.
However it is difficult to tune this sequence of decreasing stepsizes so that the method does not stop to early (before reaching the solution) or to late (thus wasting resources).

% and the method converges to a solution under a suitable choice of the stepsize sequence. However, in practice it is difficult to choose this sequence of decreasing stepsizes and the method tends to converge slowly even when the stepsizes are carefully chosen.

Another classic technique for decreasing the variance is to use the average of several $\nabla f_{i}(x_k)$ values in each iteration to get a better estimate of the full gradient $\nabla f(x)$. This is called \emph{mini-batching}, and is especially useful when multiple gradients can be evaluated in parallel.
This leads to an iteration of the form
\begin{equation}\label{eq:mini}
x_{k+1} = x_k - \gamma \frac{1}{|B_k|}\sum_{i \in B_k}\nabla f_{i}(x_k),
\end{equation} 
where  $B_k \subset \{1,\ldots, n\}$ is a set of random indices 
 and $|B_k|$ is the size of $B_k$.  \revised{When $B_k$ is sampled uniformly with replacement,}
the variance of this gradient estimator is inversely proportional to the ``batch size'' $|B_k|$, so we can decrease the variance by increasing the batch size.\footnote{\revised{If we sample without replacement the variance decreases at a faster rate (see  Section 2.7  in~\cite{lohr1999sampling}), and with $|B_k|=n$  the variance is zero.}} However, the cost of this iteration is proportional to the batch size. Thus, this form of variance reduction comes at a computational cost.

Yet another common strategy to decrease variance and improve the empirical performance of \SGD is to add ``momentum'', an extra term based on the directions used in past steps. In particular, \SGD with momentum takes the form
\begin{align} \label{eq:SGDm}
m_k & = \beta m_{k-1} +  \nabla f_{i_k}(x_k) \\
x_{k+1} &= x_k - \gamma  m_k,
\end{align}
where the momentum parameter $\beta$ is in the range $(0,1)$.
Setting $m_0 =0$ and expanding the update of $m_k$ in~\eqref{eq:SGDm} we have that  $m_k$ is a weighted average of the previous gradients,
\begin{equation} m_k = \sum_{t=0}^k \beta^{k-t} \nabla f_{i_t}(x_t).\label{eq:mk}\end{equation}
Thus $m_k$ is a weighted sum of the stochastic gradients.  
Moreover since $ \sum_{t=0}^k \beta^{k-t} = \frac{1-\beta^{k+1}}{1-\beta}$, we have that $\frac{1-\beta}{1-\beta^{k-1}}m_k$ is a weighted average of stochastic gradients. 
If we compare this with  the expression of the full gradient which is a plain average, $\nabla f(x_k) = \frac{1}{n}\sum_{i=1}^n \nabla f_i(x_k)$, we can interpret $\frac{1-\beta}{1-\beta^k}m_k$ (and $m_k)$ as an estimate of the full gradient.  This weighted sum decreases the variance but it also brings about a key problem:  Since the weighted sum~\eqref{eq:mk} gives more weight to recently sampled gradients, it does not converge to the full gradient $\nabla f(x_k)$ which is a plain average. The first variance reduced method we will see in Section~\ref{sec:VR}.\ref{sec:SAG} contours this issue by using a plain average, as opposed to any weighted average.

%
%This weighted averaging of stochastic gradients decreases the variance of $m_k$, but it 
%
% Although $m_k$ does have less variance than a single stochastic gradient  $\nabla f_{i}(x_k)$, it suffers from a key drawback:
%:   \revised{tsince $m_k$ is a weighted average, it does not converge to the full gradient  $\nabla f(x_k)$ which is a plain average. Thus $m_k$ will always have some variance.}
%he gradient estimate $m^k$ does not converge to $\nabla f(x^k).$ This is because~\eqref{eq:mk}  gives more weight to more-recently-sampled training examples.

\subsection{Modern Variance Reduction Methods}

As opposed to  classic methods that use one or more $\nabla f_i(x_k)$ directly as an approximation of $\nabla f(x_k),$ variance-reduced methods  use $\nabla f_i(x_k)$ to update an \emph{estimate} $g_k \in \R^d$ of the gradient so that $g_k \approx \nabla f(x_k)$. %See Algorithm~\ref{alg:VR}.
With this gradient estimate, we then take approximate gradient steps of the form
\begin{equation}
    \label{eq:iterateupdate}
    x_{k+1} = x_k - \gamma g_k,
\end{equation}
where $\gamma > 0$ is again the stepsize. 
%Our goal is to design $g_k$ that are substantially cheaper to calculate than the full gradient, yet more accurate estimate than the plain stochastic gradient. 
%So that $g_k$ is cheap to calculate, we will use just a few, or a \emph{mini-batch} of the $f_i(x_k)$ functions and their gradients in calculating $g_k.$ MWS: seems redundant with first sentence
%\rob{Introduce variance-reduced method as methods that maintain an estimate  $v^i  \approx \nabla f_i(x_\star)$.}
To make~\eqref{eq:iterateupdate} converge with a constant stepsize, we need to ensure that the variance of our gradient estimate $g_k$ converges to zero, that is\footnote{To be exact~\eqref{eq:variance} is not explicitly the total variance of $g_k$, but rather the trace of the covariance matrix of~$g_k$.}
%\rob{Why not $ \E{\norm{g_k - \nabla f(x_\star)}^2}  = \E{\norm{g_k}}\rightarrow 0$ instead to fit better the story?}
\begin{align}
%\E{\norm{g_k - \E{g_k}}^2}\; = \;
 \E{\norm{g_k - \nabla f(x_k)}^2}  \quad \underset{k \rightarrow \infty}{\longrightarrow} \quad 0, \label{eq:variance} 
\end{align}
where the expectation is taken with respect to all the random variables in the algorithm up to iteration $k.$ Property~\eqref{eq:variance} ensures that the \VR method will stop when reaching the optimal point. 
We take~\eqref{eq:variance} to be a defining property of variance-reduced methods and thus refer to it as the \emph{\VR property.}
% \rob{This is not entirely accurate, since if just~\eqref{eq:variance} holds ...}
 Note that  ``reduced'' variance is a bit misleading since the variance converges to zero. The property~\eqref{eq:variance} is responsible for the faster convergence of \VR methods in theory (under suitable assumptions) and in practice as we see in Figure~\ref{fig:significant}. 

%Here we give a brief review of \VR methods which that converge even with constant stepsizes. In particular~\eqref{eq:variance} holds.

%\begin{algorithm}
%	\caption{\VR: Generic {\tt V}ariance-{\tt R}educed gradient method}
%	\label{alg:VR}
%	\begin{algorithmic}[1]
%		\State {\bf Parameters:} stepsize  $\gamma >0$
%		\State {\bf Initialize:} $x_0, g_{-1} \in \R^d$
%		\For{$k = 0,1,2, \ldots$}
%			\State Sample $i_k \in \{1,\ldots, n\}$ \label{ln:sample}
%		\State $v^{i_k} \,=\;\;$update$(v^{i_k},\nabla f_i(x_k))$ \label{ln:viupdate}
%			\State $g_{k} \;\;= \;\;$update$(g_{k-1},\nabla f_i(x_k))$  \label{ln:gkupdate} % \Comment Update estimate
%			\State $x_{k+1} = x_k - \gamma g_k$  \label{ln:step} \Comment Approximate gradient step
%		\EndFor
%	\end{algorithmic}
%\end{algorithm}

\subsection{First example of a \VR method: \SGDstar}
\label{sec:star}

 One easy fix that makes the \SGD recursion in~\eqref{eq:sgd} converge without decreasing the stepsize is to simply shift each gradient by $\nabla f_{i} (x_\star)$, that is,  to use the following method
\begin{equation}\label{eq:sgdstar}
x_{k+1} = x_k - \gamma \left(  \nabla f_{i_k}(x_k) -\nabla f_{i_k} (x_\star)\right),
\end{equation}
called \SGDstar~\citep{Gorbunov20unified}. 
We note that it is unrealistic that we would know each $\nabla f_{i} (x_\star)$, but we use \SGDstar as a simple illustration of the  properties of \VR methods. Further, many \VR methods can be seen as an \emph{approximation} of the \SGDstar method; instead of relying on knowing each $\nabla f_{i} (x_\star)$, these methods use\emph{ approximations}  that converge to $\nabla f_{i} (x_\star)$. %These auxiliary vectors are also used to build an estimate $g_k \approx \nabla f(x_k) \in \R^d$ of the full gradient.

Note that \SGDstar uses an unbiased estimate of the full gradient.  Indeed, since $\nabla f(x_\star) = 0$,
\[\E{\nabla f_{i_k} (x_k) - \nabla f_{i_k} (x_\star)} = \nabla f(x_k) - \nabla f(x_\star) = \nabla f(x_k).\] 
Furthermore,  \SGDstar naturally stops when it reaches the optimal point since, for any~$i$,
\[\left. \big(\nabla f_{i}(x) -\nabla f_{i} (x_\star)\big)\right|_{x =x_\star} = 0.\]
% \rob{Also inconsistent with our defining property}
Next, we note that \SGDstar satisfies the \VR property~\eqref{eq:variance} as $x_k$ approaches $x_\star$ (for continuous $\nabla f_i$) since
\begin{align*}
 \E{\norm{g_k - \nabla f(x_k)}^2} & =  \E{\norm{\nabla f_i(x_k) - \nabla f_i(x_\star) - \nabla f(x_k)}^2}  \\
 & \leq  \E{\norm{\nabla f_i(x_k) - \nabla f_i(x_\star)}^2},
\end{align*}
where we used Lemma~\ref{lem:varbndX} with $X = \nabla f_i(x_k)-\nabla f_i(x_\star)$ and then used that $\E{\nabla f_i(x_k) - \nabla f_i(x_\star)} = \nabla f(x_k)$. This property implies that  \SGDstar  has a faster convergence rate than classic \SGD methods, as we detail in Appendix~\ref{sec:proof}.

\subsection{Faster Convergence of VR Methods}
\label{sec:assump}

In this section we introduce two standard assumptions that are used to analyze \VR methods, and discuss the speedup over classic \SGD methods that can be obtained under these assumptions. Our first assumption is Lipschitz continuity of the gradients, meaning that the gradients cannot change arbitrarily fast.
\newtcolorbox{theorybox}{colback=green!5!white,colframe=green!50!black}
\begin{theorybox}
\begin{assumption}
 \label{ass:lipschittz}
 The function $f$ is differentiable and $L$-smooth, meaning that
\begin{equation}\label{eq:Lsmooth}
\norm{\nabla f(x) - \nabla f(y)} \leq L \norm{x - y}
\end{equation}
for all $x$ and $y$ and some $0< L < \infty$. Each $f_i: \R^d \to \R$ is differentiable, $L_i$-smooth and
let $L_{\max} \eqdef \max \{L_1,\dots,L_n\}.$
\end{assumption}
\end{theorybox}
While this is typically viewed as a weak assumption,  in Section~\ref{sec:advanced} we comment on \VR methods that apply to non-smooth problems. This $L$-smoothness assumption has an intuitive interpretation for univariate functions that are twice-differentiable: it is equivalent to assuming that the second derivative is bounded by $L$, $|f''(x)| \leq L$ for every $x \in \R^d$. For multivariate twice-differentiable functions, it is equivalent to assuming that the singular values of the Hessian matrix $\nabla^2 f(x)$ are upper bounded by $L$ for every $x \in \R^d$. For the least squares problem~\eqref{eq:leastSquares}, the individual Lipschitz constants $L_i$ are given by $L_i = \norm{ a_i }^2$, while for the L2-regularized logistic regression problem~\eqref{eq:logistic}, we have $L_i = 0.25 \norm{ a_i }^2 + \lambda$.

The second assumption we consider in this section lower bounds the curvature of the functions.
\begin{theorybox}
\begin{assumption}
 \label{ass:SC}
 The function $f$ is $\mu$-strongly convex, meaning that  the function $x \mapsto f(x) - \frac \mu 2 \norm{ x }^2$ is convex  for some $\mu > 0$. Furthermore, $f_i: \R^d \to \R$ is convex for each $i=1,\ldots, n$.
 \end{assumption}
\end{theorybox}
%% A A^\top needs to be non-singular. Thus A has full row rank.
This is a strong assumption. While each $f_i$ is convex in the least squares problem~\eqref{eq:leastSquares}, the overall function $f$ is
strongly convex if and only if the design matrix $A \eqdef [a_1,\ldots, a_n]$ has full row rank.
% not strongly convex if the features are not linearly independent.
  On other hand, the L2-regularized logistic regression problem~\eqref{eq:logistic} satisfies this assumption with $\mu \geq \lambda$ due to the presence of the regularizer. As we detail in Section~\ref{sec:advanced},  it is possible to relax the strong convexity assumption as well as the assumption that each $f_i$ is convex.

An important problem class where the assumptions are satisfied are problems of the form
\begin{equation}
    \label{eq:linear_optimization}
 x_\star   \in \argmin_{x \in \R^d} \left\{ f(x) = \frac{1}{n} \sum_{i=1}^{n} \ell_i(a_i^\top x) + \frac \lambda 2 \norm{ x }^2 \right\}
\end{equation}
in the case when each  ``loss'' function $\ell_i: \R \mapsto \R$ is twice-differentiable with $\ell_i''$ bounded between $0$ and some upper bound $M$. This includes a variety of loss functions with L2-regularization in machine learning, such as least squares ($l_i(\alpha) = (\alpha - b_i)^2$), logistic regression, probit regression, Huber robust regression, and a variety of others. In this setting, for all $i$ we have $L_i \leq M \norm{ a_i }^2 + \lambda$ and $\mu \geq \lambda$.

The convergence rate of \GD under these assumptions is determined by the ratio $\kappa \eqdef L/\mu$, which is known as the \emph{condition number} of $f$. This ratio is always greater or equal to one, and when it is significantly larger than one, the level sets of the function become very elliptical which causes the iterates of the \GD method to oscillate.  This is illustrated in Figure~\ref{fig:kappa}.  In contrast, when $\kappa $ is close to 1,  \GD converges quickly.

\begin{figure}
\centering
\includegraphics[scale=0.2]{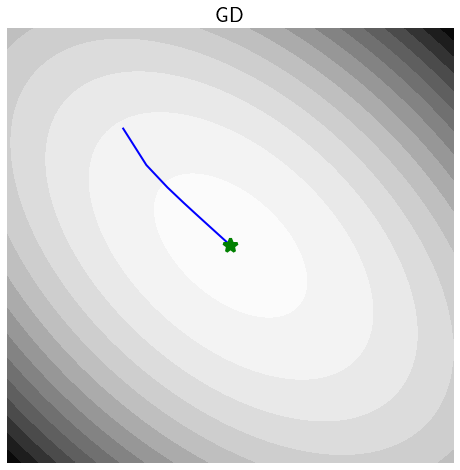}
% Use:
% level_set_correlation_0_00_grey.png
% For the figure without \GD steps
\includegraphics[scale=0.2]{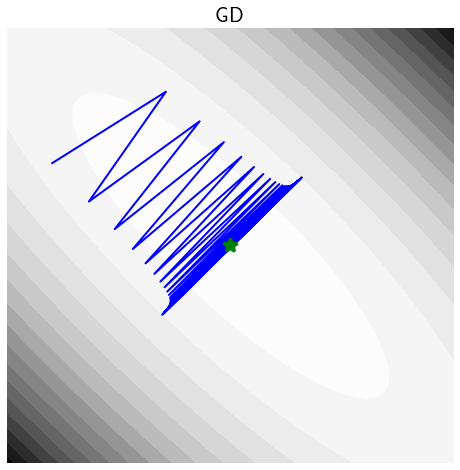}
% Use:
% level_set_correlation_0_814_grey.png
% For the figure without \GD steps
%\includegraphics[scale=0.25]{sandwich_easy.png} 
%\includegraphics[scale=0.25]{sandwich_hard.png}
\caption{Here we graph the level sets two logistic regression loss functions. The left level sets  are each of a well-conditioned logistic function with $\kappa \approx 1.$ The right level sets are of an ill-conditioned logistic functions with $\kappa \gg 1.$ }
\label{fig:kappa}
\end{figure}

Under Assumptions~\ref{ass:lipschittz} and~\ref{ass:SC}, \VR methods converge at a \emph{linear rate}. We say that the  function values $\{f(x_k)\}$ of a randomized method  converge linearly (in expectation) at a rate of $0 < \rho \leq 1$ if there exists a constant $C>0$ such that
%\begin{equation}\label{eq:expconv}
%\E{\norm{x_{k} -x_\star}^2} \; \leq \; (1-\rho)^k C,
%\end{equation}
%or
\begin{equation}\label{eq:expconvf}
\E{f(x_k)} - f(x_\star)  \leq  (1-\rho)^k C = \cO(\exp(-k\rho)), \quad \forall k.
\end{equation}
%Both~\eqref{eq:expconv} and~\eqref{eq:expconvf} are equally good under our assumptions, since by plugging in $y =x_\star$ into~\eqref{eq:Lsmooth} and~\eqref{eq:muconvex} we have that
%\[\frac{L}{2}\norm{x- x_\star}^2 \; \leq \;  f(w) -f(x_\star) \; \leq\; \frac{\mu}{2}\norm{x- x_\star}^2,\]
%and so convergence of $\norm{x- x_\star}$ implies the convergence of the function values $f(x)-f(x_\star)$ and vice versa. \francis{Is the discussion above really needed? We could choose of way of measuring and stick to it.}
%\rob{Agreed.}

This is in contrast to classic \SGD methods that only rely on an unbiased estimate of the gradient in each iteration, which under these assumptions can only obtain the sublinear rate 
\[
\E{f(x_k)} - f(x_\star)  \leq  \cO(1/k).
\]
Thus, classic \SGD methods become slower the longer we run them, while \VR methods continue to cut the error by at least a fixed fraction in each step.

%We can use this convergence~\eqref{eq:expconvf} to bound the number of iterations needed to reach a given tolerance $\varepsilon >0$ on the error. That is, we have as a consequence of~\eqref{eq:expconvf} that if
As a consequence of~\eqref{eq:expconvf}, we can determine the number of iterations needed to reach a given tolerance $\varepsilon >0$ on the error as follows
\begin{equation}\label{eq:itercomplex}
 k \; \geq \; \frac{1}{\rho}\log\left(\frac{C}{\varepsilon} \right), \quad \mbox{then} \quad \E{f(x_k)} - f(x_\star) \leq \varepsilon.\end{equation}
The smallest $k$ satisfying this inequality is known as the \emph{iteration complexity} of the algorithm. Below we give the iteration complexity and the cost of one iteration in terms of $n$ for the basic variant of \GD, \SGD, and \VR methods:\\ 
%\francis{Should we say here that we measure time in terms of individual gradient evaluations?}
\begin{center}
\begin{tabular}{c|c|c}
Algorithm & \# Iterations & Cost of 1 Iteration\\
\hline
\GD & $\cO(\kappa\log(1/\varepsilon))$ & $\cO(n)$\\
\SGD & $\cO(\revised{\kappa_{\max}}(1/\varepsilon))$ & $\cO(1)$ \\
\VR & $\cO((\kappa_{\max} + n)\log(1/\varepsilon))$ & $\cO(1)$
\end{tabular}
\end{center}

The total runtime of an algorithm is given by the product of the iteration complexity and the iteration runtime.
Above we have used $\kappa_{\max} \eqdef (\max_i L_i)/\mu$. Note that $\kappa_{\max} \geq \kappa$,  thus the iteration complexity of \GD is smaller than that of the \VR methods.\footnote{\revised{In Section~\ref{sec:advanced} we discuss how non-uniform sampling within VR methods leads to a faster rate, depending on the mean $\bar{L} \eqdef \tfrac{1}{n}\sum_i L_i$  rather than on the maximum $\max_i L_i$.}} But the  \VR methods are superior in terms of total runtime since each iteration of gradient descent costs $n$-times more than an iteration of a \VR method\footnote{And since $L_{\max} \leq n L.$}. Classic \SGD methods have the advantage that their runtime and their convergence rate does not depend on $n$, but it does have a much worse dependency on the tolerance $\varepsilon$ which explains \SGD's poor performance when the tolerance is small.
%but they lead to poor performance for many problems due to their much-worse dependence on $\varepsilon$.
\footnote{
%\revised{The \SGD result requires an additional assumption that the gradient estimates are bounded in some way. For example, this rate can be shown under the assumption that $\EE{i}{\norm{\nabla f_i(x_*)}^2}$ is bounded~\citep{nemirovski2009robust}.}
\revised{We have omitted an additional term for the \SGD iteration complexity of the form $\cO(\sigma^2/\mu\varepsilon)$, where $\sigma^2 \geq \EE{i}{\norm{\nabla f_i(x_*)}^2}  $, see Theorem 2.1 in~\cite{NeedellWard2015}}
}

In Appendix~\ref{sec:theory} we give a simple proof showing that  the \SGDstar method has the  same iteration complexity as the \VR methods.

\section{Basic Variance-Reduced Methods}
\label{sec:VR}
The first wave of variance-reduced methods that achieve the convergence rate from the previous section started with the stochastic average gradient (\SAG) method~\citep{roux2012stochastic,SAG}. This was followed shortly after by the stochastic dual coordinate ascent (\SDCA)~\citep{UCDC,Shalev-Shwartz2013c}, \MISO~\citep{MISO}, stochastic variance-reduced gradient (\SVRG/\StwoGD)~\citep{mahdavi2013mixedgrad,Johnson2013,S2GD,zhang2013linear}, and \SAGA \revised{(stochastic average gradient ``am\'{e}lior\'{e}'')}~\citep{SAGA} methods. In this section we present several of these original methods, while Section~\ref{sec:advanced} covers more recent methods that offer improved properties in certain settings over these original methods.

\subsection{Stochastic Average Gradient (SAG)} \label{sec:SAG}
The first \VR method is based on mimicking the structure of the full gradient. Since the full gradient $\nabla f(x)$ is a plain average of the  $\nabla f_i(x)$ gradients, our estimate $g_k$ of the full gradient should be an average of estimates of the $\nabla f_i(x)$ gradients.
%However, this estimate does not have the \VR property:
%The solution to this is to maintain an average of the previous gradients such the gradient of each training example gets an equal weight.
 This idea leads us to our first variance-reduced method: the \emph{stochastic average gradient} (\SAG) method.

The stochastic average gradient (\SAG) method~\citep{roux2012stochastic,SAG} is a stochastic variant of the earlier incremental aggregated gradient (\IAG) method~\citep{blatt2007convergent}. The idea behind \SAG is to maintain an estimate $v_k^i \approx \nabla f_i(x_k)$ for each data point $i $. We then use the average of the $v_k^i$ values as our estimate of the full gradient, that is
\begin{equation}\label{eq:SAG}
 \bar{g}_k = \frac{1}{n} \sum_{j=1}^n v_k^j \approx\frac{1}{n} \sum_{j=1}^n \nabla f_j(x_k) = \nabla f(x_k) .
\end{equation}
At each iteration, \SAG samples $i_k \in \{1,\ldots, n\}$ and updates the $v_k^j$ using
\begin{equation}\label{eq:SAGupdate}
v_{k+1}^j  =  \begin{cases} \nabla f_{i_k}(x_k) & \quad \mbox{if }j=i_k\, , \\
v_{k}^j & \quad \mbox{if }j\neq i_k\, ,
\end{cases}
\end{equation}
where each $v^0_i$ might be initialized to zero or to an approximation of $\nabla f_i(x_0)$. As we approach a solution $x_\star$, each $v^i$ converges to $\nabla f_i(x_\star)$ which gives us the \VR property~\eqref{eq:variance}.

To implement \SAG efficiently, we need to take care in computing $\bar{g}_k$ using~\eqref{eq:SAG}, since this requires summing up $n$ vectors in $\R^d$, and since $n$ can be very large, computing this sum can be very costly. Fortunately we can avoid computing this summation from scratch every iteration since only one $v_{k}^i$  term will change in the next iteration. That is, suppose we sample the index $i_k$ on iteration $k$. It follows from~\eqref{eq:SAG} and~\eqref{eq:SAGupdate} that
\begin{align}
\bar{g}_k &=  \frac{1}{n} \sum_{j=1, j \neq i_k}^n v_k^j + \frac{1}{n} v_k^{i_k} \nonumber \\
& = \frac{1}{n} \sum_{j=1, j \neq i_k}^n v_{k-1}^j + \frac{1}{n} v_k^{i_k}  && \mbox{(Since $v_{k-1}^j = v_{k}^j$ for all $j \neq  i_k$)}\nonumber\\
&= \bar{g}_{k-1}- \frac{1}{n} v_{k-1}^{i_k} + \frac{1}{n} v_k^{i_k}. && \mbox{(Plus and minus $\frac{1}{n} v_{k-1}^{i_k}$)} \label{eq:bargup}
\end{align}
Since the $v_{k-1}^{j} $ are simply copied over to $v^j_k$, when implementing \SAG we can simply store one vector $v^j$ for each $j$. This implementation is illustrated in Algorithm~\ref{alg:SAG}. 

\begin{algorithm}
\begin{algorithmic}[1]
\State {\bf Parameters:} stepsize $\gamma >0$
\State {\bf Initialize:} $x_0, v^i = 0 \in \R^d$ for $i=1,\ldots, n$ 
\For{$k =1,\ldots, T-1$}
\State  Sample $i_k \in \{1,\ldots, n\}$
\State $\displaystyle \bar{g}_k = \bar{g}_{k-1} -\frac{1}{n} v^{i_k}
$ \label{ln:Gk1}
\State $\displaystyle v^{i_k} = \nabla f_{i_k}(x_k)$ \label{ln:gcalculate}
\State $\displaystyle \bar{g}_k = \bar{g}_{k} +\frac{1}{n} v^{i_k}
$ \label{ln:Gk2}
\State $\displaystyle x_{k+1} = x_k - \gamma \bar{g}_k$ \label{ln:xstep}
\EndFor  
\State {\bf Output:} $x_{T}$
\end{algorithmic}
\caption{\SAG: {\tt S}tochastic {\tt A}verage {\tt G}radient method}
\label{alg:SAG}
\end{algorithm}

The \SAG method was the first stochastic methods to enjoy linear convergence with an iteration complexity of $\cO((\kappa_{\max} + n)\log(1/\varepsilon))$, using a stepsize of
%\begin{equation} \label{eq:SAGconv}
% \E{\norm{x_k - x_\star}^2} \; =\; \left( 1- \min\left\{\frac{\mu}{16L_{\max}}, \; \frac{1}{8 n} \right\}\right)^k \norm{x_0 - x_\star}^2,
%\end{equation}
 $\gamma  = \cO\left(1/L_{\max}\right).$  This linear convergence can be seen in Figure~\ref{fig:significant}. Note that since an $L_{\max}$-smooth function is also $L'$-smooth for any $L' \geq L_{\max}$, this method obtains a linear convergence rate for \emph{any} sufficiently small stepsize. This is in contrast to classic \SGD methods, which only obtain sublinear rates and only under difficult-to-tune-in-practice decreasing stepsize sequences.
 
At the time, the linear convergence of \SAG was a remarkable breakthrough given that \SAG only computes a single stochastic gradient (processing a single data point) at each iteration.
However, the convergence proof by~\cite{SAG} is notoriously difficult, and relies on computer verified steps. What specifically makes \SAG hard to analyze is that $\overline{g}_k$ is a \emph{biased} estimate of the gradient. %This is because at each iteration we sample an index $i$ independently from the previous iteration, in expectation $g_k$ is the average of $n$ full gradients evaluated on past iterations.
%\rob{Finito/Miso~\citep{MISO,Finito} is a closely related method to \SAG that in addition makes use of iterate averaging...
%}
Next we introduce the \SAGA method, a variant of \SAG that uses the concept of  \emph{covariates} to make an \emph{unbiased} variant of the \SAG method that has similar performance but is easier to analyze.

% Indeed, since at each iteration we sample an index $i$ independently from the previous iteration, we have that
%\[\E{g_k} = \frac{1}{n} \sum_{i=1}^n \E{g_k^i} =\frac{1}{n} \sum_{i=1}^n \nabla f (x_{t_i}) \]

%, \SAG was to have an iteration complexity of 
%$O\left(\left(\frac{L_{\max}}{\mu}+n\right) \log\left(\frac{1}{\varepsilon} \right)\right)$.

\subsection{SAGA}

%Let $i_k \sim \frac{1}{n}.$ 
%Starting from the simple unbiased estimate $\nabla f_{i_k} (x_k)$ where $i_k \sim \frac{1}{n}$, 
\revised{A common way to reduce the variance of the basic unbiased estimate $\nabla f_{i_k}(x_k)$ is by using what is known as \emph{covariates} (or ``control variates''). 
Let $v^i \in \R^d$ be a vector for  $i =1,\ldots, n.$ Using these vectors we can rewrite our full gradient as
\begin{align}
\nabla f(x) 
& =  \frac{1}{n} \sum_{i=1}^n (\nabla f_i(x)-v^i +v^i) \nonumber \\
&=  \frac{1}{n} \sum_{i=1}^n \left(\nabla f_i(x)-v^i +\frac{1}{n}\sum_{j=1}^n v^j\right) \nonumber  \\ 
& \eqdef  \frac{1}{n} \sum_{i=1}^n \nabla  f_i(x,v), \label{eq:probz}
\end{align}
where $\nabla f_i(x,v) \eqdef \nabla f_i(x)-v^i +\frac{1}{n}\sum_{j=1}^n v^j.$
Now we can build an unbiased estimate of the full gradient $\nabla f(x)$ by sampling a single $\nabla f_i(x,v)$ uniformly for $i \in \{1,\ldots, n\}.$
  That is, we can solve~\eqref{eq:optimization_problem} by  applying the \SGD method with the gradient estimate
\begin{align}
 g_k &= \nabla f_{i_k}(x_k,v) % \nonumber \\
 \; = \nabla f_{i_k}(x_k)-v^i +\frac{1}{n}\sum_{j=1}^n v^j.\label{eq:gkcov}
\end{align} }

\revised{To see the effect of the choice of the $v^i$'s on the variance of $g_k$, substituting $g_k = \nabla f_{i_k}(x_k,v) $ and using  $\EE{i \sim \frac{1}{n}}{v^i} = \tfrac{1}{n}\sum_{j=1}^n  v^j $ in~\eqref{eq:variance}  gives
\begin{align}
\eqref{eq:variance} & =\E{\norm{ \nabla f_i(x_k)-v^i +\EE{i \sim \frac{1}{n}}{v^i -\nabla f_i(x_k)} }^2}  \nonumber \\
 &\leq \E{\norm{\nabla f_i(x_k)-v^i}^2},\label{eq:covariatesvar}
\end{align}
where we used Lemma~\ref{lem:varbndX} with $X = \nabla f_i(x_k)-v^i$. This bound~\eqref{eq:covariatesvar} shows us that
we obtain the \VR property~\eqref{eq:variance} if
% we can minimize the variance by simply choosing 
$v^i$ approaches $\nabla f_i(x_k)$ as $k$ grows. This is why we refer to the $v^i$'s as \emph{covariates}.
We are free to choose any $v^i$, so we can choose them to reduce the variance. }

\revised{ As an example, the \SGDstar method~\eqref{eq:sgdstar} also  implements this approach with $v^i= \nabla f_i(x_{\star})$. But again, this is not practical since often we do not know $\nabla f_i(x_{\star})$. A more practical choice for $v^i$ is the gradient $\nabla f_i(\bar{x}_i)$ around a point $\bar{x}_i\in \R^d$ that we do know. 
%This is exactly the choice of covariates that gives rise to the \SAGA method~\citep{SAGA}. 
 \SAGA  uses a reference point $\bar{x}_i \in \R^d$ for each  function $f_i$ and uses the  covariate $v^i = \nabla f_i(\bar{x}_i)$
where each $\bar{x}_i $ will be the last point for which we evaluated $\nabla f_i(\bar{x}_i)$.}
Using these covariates we can build a gradient estimate following~\eqref{eq:gkcov} which gives
\begin{equation}\label{eq:SAGAgk}
g_k \;= \; \nabla f_{i_k}(x_k)-\nabla f_{i_k}(\bar{x}_{i_k}) +\frac{1}{n}\sum_{j=1}^n \nabla f_{j}(\bar{x}_j).
\end{equation} 
To implement \SAGA, instead of storing the $n$ reference points $\bar{x}_i$ we can store the gradients $ \nabla f_i(\bar{x}_i)$
%; the stochastic gradients evaluated at their respective reference points. 
That is, let $v^j =  \nabla f_{j}(\bar{x}_j)$ for $j \in \{1,\ldots, n\}$, and similar to \SAG we update $v^j$ of one random gradient in each iteration. We formalize the \SAGA method as Algorithm~\ref{alg:SAGA}, which is similar to the implementation of \SAG (Algorithm~\ref{alg:SAG}) except now we store the previously known gradient of $f_{i_k}$ in a dummy variable $ v^{\mbox{\scriptsize old}}$ so that we can then form the unbiased gradient estimate~\eqref{eq:SAGAgk}.
\begin{algorithm}
\begin{algorithmic}[1]
\State {\bf Parameters:} stepsize $\gamma >0$ 
\State {\bf Initialize:} $x_0, v^i = 0 \in \R^d$ for $i=1,\ldots, n$ 
\For{$k =1,\ldots, T-1$}
\State  Sample $i_k \in \{1,\ldots, n\}$
\State $v^{\mbox{\scriptsize old}} =v^{i_k}$ 
\State $\displaystyle v^{i_k} = \nabla f_{i_k}(x_k)$ \label{ln:gcalculate2}
%\State $\displaystyle \bar{g}_k = \bar{g}_{k} +\frac{1}{n} v^{i_k}
%$ \label{ln:Gk22}
\State $\displaystyle x_{k+1} = x_k - \gamma \left( v^{i_k} -  v^{\mbox{\scriptsize old}}+\bar{g}_k \right)$ \label{ln:xstep2}
\State $\displaystyle \bar{g}_k = \bar{g}_{k-1} +\frac{1}{n} v^{i_k} -\frac{1}{n}v^{\mbox{\scriptsize old}}
$ \label{ln:Gk12}
\EndFor  
\State {\bf Output:} $x_{T}$
\end{algorithmic}
\caption{\SAGA}
\label{alg:SAGA}
\end{algorithm}

The \SAGA method has an iteration complexity of $\cO((\kappa_{\max} + n)\log(1/\varepsilon))$ using a stepsize of $\gamma = \cO(1/L_{\max})$ as in \SAG, but with a much simpler proof. However, as with \SAG, the \SAGA method needs to store the auxiliary vectors $v^i \in \R^d$ for $i =1,\ldots, n$ which amounts to an $\cO(nd)$ storage. This can be infeasible when both $d$ and $n$ are large. We detail in Section~\ref{sec:practical} how we can reduce this memory requirement for common models like regularized linear models~\eqref{eq:linear_optimization}.

\revised{When the $n$ auxiliary vectors can be stored in memory,  \SAG and \SAGA tend to perform similarly. 
When this  memory requirement is too high, the \SVRG method that we review next is a good alternative. 
The \SVRG method achieves the same convergence rate, and is often nearly as fast in practice, but \emph{only} requires $\cO(d)$ memory for general problems.}
%This is an issue shared by a family of variance-reduced methods called the \emph{memorization methods}, that we present next. In Section~\ref{sec:practical} we show that this large memory footprint can be alleviated when training linear models. 

\subsection{SVRG}

Prior to \SAGA, the first works to use covariates used them to address the high memory required of \SAG~\citep{Johnson2013,mahdavi2013mixedgrad,zhang2013linear}. These works build covariates based on a fixed reference point $\bar{x} \in \R^d$  at which we have already computed the full gradient $\nabla f(\bar{x})$. By storing $\bar{x}$ and $\nabla f(\bar{x})$, we can implement the update~\eqref{eq:SAGAgk} using $\bar{x}_{j} = \bar{x}$ for all $j$ without storing the individual gradients $\nabla f_{j}(\bar{x})$. In particular, instead of storing these vectors, we compute $\nabla f_{i_k}(\bar{x})$ in each iteration using the stored reference point $\bar{x}$. 
%This requires two gradient evaluations per iteration, and also requires us to occasionally update the reference point $\bar{x}$ and compute its full gradient $\nabla f(\bar{x})$. 
Originally presented under different names by different authors, this method has come to be known as the stochastic variance-reduced gradient (\SVRG) method, following the naming of~\citep{Johnson2013,zhang2013linear}.
% If we had a point $\bar{x}$ that was close enough to $x_k$, and if $\nabla z_i (x_k) = \nabla f_i(\bar{x})$ then by the continuity of the gradient and~\eqref{eq:covariatesvar}, we would have a small variance. This is exactly the main idea behind the \SVRG (stochastic variance-reduced gradient)~\citep{Johnson2013,zhang2013linear} method. who was the first to formalize  variance reduction for gradient techniques. \francis{I added the last sentence because we present \SAGA before \SVRG}
%The covariates are built using fixed reference point $\bar{x} \in \R^d$ for which we have already computed the full gradient $\nabla f(\bar{x}).$ If $\bar{x}$ is not far from $x_k$ then by the  $L_i$-Lipschitz continuity of the gradient (Assumption~\ref{ass:lipschittz}) we have that $\nabla f_i(x_k)$ and $\nabla f_i(\bar{x})$ are also close, specifically
%\[ \norm{\nabla f_i(x_k) - \nabla f_i(\bar{x})} \; \leq \; L_i \norm{x_k -\bar{x}}.\]
%Consequently by setting 
%$\nabla z_i(x_k) = \nabla f_i(\bar{x})$ we have that
% the variance~\eqref{eq:covariatesvar} is bounded by the distance between $\bar{x}$ and $x_k$. The closer $\bar{x}$ is to $x_k$, the smaller the variance. 

 We formalize the \SVRG method in Algorithm~\ref{alg:f-SVRG}. 
 \revised{Using~\eqref{eq:covariatesvar} we have that the variance of the gradient estimate $g_k$ is bounded by
 \begin{align*} \E{\norm{g_k-\nabla f(x_k)}^2} & \leq  \E{\norm{\nabla f_i(x_k)-\nabla  f_i(\bar{x})}^2}\\
 & \leq L_{\max}^2\norm{{x_k} - \bar{x}}^2,\end{align*}
 where the second inequality uses the $L_i$-smoothness of each $L_i$\footnote{\revised{When each $f_i$ is also convex, we can derive the bound 
  \[\mathbb{E}[\|\nabla f(\bar x_{s-1})-\nabla f(x_k)\|^2] \leq 4L_{\max}(f(x_k)-f(x^*) + f(\bar x_{s-1})-f(x^*))\] 
  using analogous  proof to Lemma~\ref{lem:expsmooth}. This bound on the variance is key to proving a good convergence rate for \SVRG in the convex setting}.}.}
Notice that the closer $\bar{x}$ is to $x_k$, the smaller the variance of the gradient estimate.

To make the \SVRG method work well, we need to trade off the cost of updating the reference point $\bar{x}$ frequently, and thus having to compute the full gradient, with the benefits of decreasing the variance. To do this, the reference point is updated every $t$ iterations to be a point close to $x_k$: see line 11 of Algorithm~\ref{alg:f-SVRG}
That is, the \SVRG method has two loops: one outer loop in $s$  where the reference gradient $\nabla f(\bar{x}_{s-1})$ is computed (line 4), 
and one inner loop where the reference point is fixed and the inner iterates $x_k$ are updated (line 10) according to stochastic gradient steps using~\eqref{eq:gkcov}.

\revised{In contrast to \SAG and \SAGA, \SVRG  requires $\cO(d)$ memory \emph{only}.
The downsides of \SVRG are 1) we have an additional parameter $t$, the length of the inner loop, which needs to be tuned and 2)   two gradients are computed per iteration and the  full gradient needs to be computed every time the reference point is changed. }  

%Closing this gap between what works well in practice and what we can prove in theory is an active direction of research. 
%We comment on this and also several variants of \SVRG and methods that rely resetting a reference point in Section~\ref{sec:resetting}.

% We give some of the results of this theory and the current gaps between practice in Section~\ref{sec:Asvrg}

\begin{algorithm}%[!h]
    \begin{algorithmic}[1]
        \State \textbf{Parameters} stepsize $\gamma>0$ 
        \State \textbf{Initialization} $\bar{x}_0 = x_0 \in \mathbb{R}^d$
        \For {$s=1, 2,\dots$}\vskip 1ex
            \State Compute and store $\nabla f(\bar{x}_{s-1})$
            \State $x_0 = \bar{x}_{s-1}$ \label{ln:inneriterreset}%Set $w_s = x_s^0 = x_{s-1}^m$ 
            \State Choose the number of inner-loop iterations $t$
            \For {$k=0, 1,\dots, t-1$}\vskip 1ex
            \State Sample $i_k \in \{1,\ldots, n\}$
%            \State $g_k = \nabla f_{i_k}(x_s^k)- \nabla f_{i_k}(\bar{x}_{s-1}) + \nabla f(\bar{x}_{s-1})$ 
             \State $g_k = \nabla f_{i_k}(x_k)- \nabla f_{i_k}(\bar{x}_{s-1}) + \nabla f(\bar{x}_{s-1})$            
            % 	 	\Comment Stochastic Gradient estimate. \label{com:sto_step}
            \State $x_{k+1} = x_k - \gamma g_k$
        \EndFor
        \State $\bar{x}_s = x_{t}$.
%        \State {\bf Option theory:} $\bar{x}_s = \sum_{k=0}^{m-1} p_k x_k$ \label{ln:theory}
  %              \State {\bf Option practical:} $\bar{x}_s =x_m$ \label{ln:practical}
        \EndFor
    \end{algorithmic}
    \caption{\SVRG: {\tt S}tochastic {\tt V}ariance-{\tt R}educed {\tt G}radient method}
    \label{alg:f-SVRG}
\end{algorithm}

%\peter{I think we should just describe loopless \SVRG. it is simpler to write, simpler to understand, has better theory ($t$ does not depend on condition number) and is more practical. We can always say how \SVRG can be obtained in words. Fits better with \SAG and \SAGA. I can do this... Yes, \SVRG is historically important, and we can say that, but we do not and need not follow historical development in this review. } 

%\paragraph{Implementation.} Typically when implementing \SVRG, only the full gradient $\nabla f(\bar{x}_{s-1})$ and the reference points $\bar{x}_{s-1} \in \R^d$ are stored. The stochastic gradients $\nabla f_{i_k}(x_k)$ and $\nabla f_{i_k}(\bar{x}_{s-1})$ are computed at each inner iteration. In particular $\nabla f_{i_k}(x_k)$ is not stored or re-used in any way, which is different than the \SAG method which stores and uses every stochastic gradient to update the $g_k$ estimate~\eqref{eq:SAGupdate}. 

%\paragraph{Downside.} One of the main issues with \SVRG is that it requires computing the full gradient every time the reference point is changed, which is very costly when $n$ is large (and impossible when $n$ is infinite). 

 \cite{Johnson2013} showed that \SVRG has  iteration complexity  $\cO((\kappa_{\max} + n)\log(1/\varepsilon))$, similar to \SAG and \SAGA. This was shown assuming that the number of inner iterations $t$ is sampled uniformly from $\{1,\dots,m\}$, where a complex dependency must hold between $L_{\max}$, $\mu$, the stepsize $\gamma$, and  $t$.
In practice, \SVRG tends to work well by using  $\gamma = \cO(1/L_{\max})$ and inner loop length $t =n$, which is the setting we used in Figure~\ref{fig:significant}. 
%In practice, \SVRG tends to work well by using a stepsize $\gamma = \cO(1/L_{\max})$, inner loop length $m =n$, and setting the reference to the last iterate of the inner loop, see {\bf Option practical} on line~\ref{ln:practical} in Algorithm~\ref{alg:f-SVRG}. This is the setting we used in Figure~\ref{fig:significant}.

There are now many variations on the original \SVRG method. For example, there are variants that use alternative distributions for $t$~\citep{S2GD} and variants that allow stepsizes of the form $\cO(1/L_{\max})$~\citep{HofmannLLM152015,kulunchakov2019estimate,loopless}.
There are also variants that use a mini-batch approximation of $\nabla f(\bar{x})$ to reduce the cost of these full-gradient evaluations, and that grow the mini-batch size in order to maintain the \VR property~\citep{Konecny:wastinggrad:2015,frostig2015competing}. And there are variants that repeatedly update $g_k$ in the inner loop according to~\citep{nguyen17b}
\begin{equation}
\label{eq:SARAH}
g_k = \nabla f_{i_k} (x_k) -\nabla f_{i_k} (x_{k-1}) +g_{k-1},
\end{equation}
which provides a more local approximation.  \revised{Using this continuous update variant~\eqref{eq:SARAH} has shown to have distinct advantages in minimizing nonconvex functions, as we briefly discuss in Section~\ref{sec:nonconvex}}.
  Finally, note that \SVRG can use the values of $\nabla f(\bar{x}_s)$   to help decide when to terminate the algorithm.

\subsection{SDCA and Variants}

% and gives an unbiased estimate $g_k$ of the gradient like \SVRG, .

%\schmidt{Shorten!!!!!! Probably just want to say that they looked at (linear)+L2, take the dual and apply CD, it is nice because you can pick the stepsize but still requires same memory as \SAG/SAGA, and later method was dual-free.}
%
%\francis{agreed}

A drawback of \SAG and \SVRG is that their stepsize depends on $L_{\max}$, which may not be known for some problems. \revised{One of the first  \VR methods (pre-dating \SVRG) was the stochastic dual coordinate ascent (\SDCA) method~\citep{Shalev-Shwartz2013c} which extended recent work on coordinate descent methods to the finite-sum problem.\footnote{The modern interest in coordinate ascent/descent methods began with~\cite{nesterov2012efficiency}, which considered coordinate-wise gradient descent with randomly-chosen coordinates, and included a result showing linear convergence for $L$-smooth strongly convex functions.
 This led to an explosion of work on the problem, as for many problem structures we can very-efficiently compute coordinate-wise gradient descent steps~\citep{UCDC}% (arXiv:1107.2848, 2011).
 %which in fact appeared online before the \SDCA method (arXiv:1209.1873, 2012).
  }}

\revised{The intuition behind \SDCA, and its variants, is that the coordinates of the gradient  provide a naturally variance-reduced estimate of the gradient. That is, let $j \in \{1,\ldots, d\}$ and let
$\nabla_j f(x) \eqdef \tfrac{\partial f(x)}{\partial  x_j} e_j$ be the \emph{coordinate-wise} derivative of $f(x),$ where $e_j \in \R^d$ is the $j$th unit coordinate vector.
An important feature of coordinate-wise derivatives is that $\nabla_j f(x_\star) =0$, since we know that $\nabla f(x_\star)=0$. This is unlike the derivative for each data point $\nabla f_j$ that may be different than zero at $x_\star.$ Due to this feature, we have that 
\begin{equation}
 \norm{\nabla f(x) - \nabla_j f(x)}^2  \quad \underset{x \rightarrow x_\star}{\longrightarrow} \quad 0,
\end{equation}
and thus  coordinate-wise derivative satisfies the \VR property~\eqref{eq:variance}.
Furthermore, we can also use $\nabla_jf(x)$ to build an unbiased estimate of $\nabla f(x).$ For instance, let $j$ be a random index sampled uniformly on average from $ \{1,\ldots, d\}.$ Thus for any given $i \in \{1,\ldots, d\}$ we have that $\Prob{j =i} = \tfrac{1}{d}.$ Consequently $d \times \nabla _j f(x)$ is an unbiased estimate of $\nabla f(x)$ since
\[\E{d \nabla _j f(x)} = d\sum_{i=1}^d \Prob{j =i}\frac{\partial  f(x)}{\partial  x_i} e_i  = \sum_{i=1}^d\frac{\partial  f(x)}{\partial  x_i} e_i = \nabla f(x). \]
Thus $\nabla_j f(x)$ has all the favorable properties we would like for a variance reduced estimate of the full gradient without using covariates. The downside of using this coordinate-wise gradient is that for our sum-of-terms problem~\eqref{eq:optimization_problem} it is expensive to compute. This is because computing $\nabla_j f(x)$ requires a full pass over the data since 
\[ \nabla _j f(x) = \frac{1}{n}\sum_{i=1}^n \nabla_j f_i(x). \]
So it would seem that using coordinate-wise derivatives is incompatible with the structure of our sum-of-terms problem.
 Fortunately though, we can often rewrite our original problem~\eqref{eq:optimization_problem} in what is known as a \emph{dual formulation} where coordinate-wise derivatives can make use of the inherit structure.}

\revised{To illustrate, the dual formulation of the L2-regularized linear models of the form~\eqref{eq:linear_optimization} is given by 
\begin{equation}\label{eq:dual}
 v^\star   \in \argmax_{v \in \R^n} \left\{ \frac 1 n\sum_{i=1}^n -\ell_i^\star(-v^i) - \frac{\lambda}{2}\norm{\frac{1}{\lambda n}\sum_{i=1}^n v^i a_i}^2\right\},
 \end{equation}
 where $\ell_i^\star(v) \eqdef \sup_x \{\dotprod{x,v} - f(x) \}$ is the \emph{convex conjugate} of $\ell_i$. We can recover the $x$ variable of our original problem~\eqref{eq:linear_optimization} using the mapping
\[x = \frac{1}{\lambda n}\sum_{i=1}^n v^ia_i.\] 
Consequently, plugging in the solution $v^\star$ to~\eqref{eq:dual} in the right hand side of the above gives $x_\star$, the solution of~\eqref{eq:linear_optimization}.}

 \revised{Notice that this dual problem has $n$ real variables $v^i\in \R$, one for each training example. Furthermore, each \emph{dual} loss function $\ell_i^\star$ in
~\eqref{eq:dual} is a function of a single $v^i$ only. That is, the first term in the loss function is \emph{separable} over coordinates. It is this separability over coordinates, combined with the simple form of the second term, that allows for an efficient implementation of a coordinate ascent method.\footnote{We call it ``coordinate ascent'' instead of ``coordinate descent'' since~\eqref{eq:dual} is a maximization problem. }
 Indeed, ~\cite{Shalev-Shwartz2013c} showed that coordinate ascent on this dual problem has an iteration complexity of $\cO((\kappa_{\max} + n)\log(1/\varepsilon))$, similar to \SAG, \SAGA and \SVRG.\footnote{This iteration complexity is in terms of the duality gap.  Related results for certain problem structures include~\citet{strohmer2009randomized,collins2008exponentiated}.} The iteration cost and algorithm structure are also quite similar: by keeping track of the sum $\sum_{i=1}^n v^ia_i$ to address the second term in~\eqref{eq:dual},  each dual coordinate ascent iteration only needs to consider a single training example and the cost per iteration is independent of $n$.  Further, we can use a one-dimensional line-search to efficiently compute a stepsize that maximally increases the dual objective as a function of one $v^i$. 
%Unlike the other early \VR works, \SDCA allowed the use of line search in order to set the stepsize. 
 This means that the fast worst-case runtime of \VR methods can be achieved with no knowledge of $L_{\max}$ or related quantities.}

Unfortunately, the \SDCA method also has several disadvantages. First, it requires computing the convex conjugates $\ell_i^\star$ rather than simply gradients. We do not have an equivalent of automatic differentiation for convex conjugates, so this may increase the implementation effort. More recent works have presented ``dual-free'' \SDCA methods that do not require the conjugates and instead work with gradients~\citep{dfSDCA}. However, it is no longer possible to track the dual objective in order to set the stepsize in these methods.
Second, while \SDCA only requires $\cO(n+d)$ memory for problem~\eqref{eq:linear_optimization}, \SAG/\SAGA also only requires $\cO(n+d)$ memory for this problem class (see the next section). Variants of \SDCA that apply to more general problems have the $\cO(nd)$ memory of \SAG/\SAGA since the $v^i$ become vectors with $d$-elements.
A final subtle disadvantage of \SDCA is that it implicitly assumes that the strong convexity constant $\mu$ is equal to $\lambda$. For problems where $\mu$ is greater than $\lambda$, \peter{Point to the work of Lin Xiao.} the primal \VR methods often significantly outperform \SDCA.

\section{Practical Considerations}\label{sec:practical}

In order to implement the basic \VR methods and obtain a reasonable performance, several implementation issues must be addressed. In this section, we discuss several issues that are not addressed above. 

\subsection{Setting the stepsize for SAG/SAGA/SVRG}

While we can naturally use the dual objective to set the stepsize for \SDCA, the theory for the primal \VR methods \SAG/\SAGA/\SVRG relies on stepsizes of the form $\gamma = \cO(1/L_{\max})$. Yet in practice one may not know $L_{\max}$ and better performance can often be obtained with other stepsizes. 

One classic strategy for setting the stepsize in full-gradient descent methods is the  \emph{Armijo line-search}~\citep{armijo1966}.
Given a current point $x_k$ and a search direction $g_k,$ the Armijo line-search for a $\gamma_k$ that is on the line $\gamma_k \in \{\gamma \; : \; x_k + \gamma g_k \}$ and such that gives a sufficient decrease of the function
\begin{equation}
 f(x_k + \gamma_k g_k) < f(x_k) -c \gamma_k \norm{\nabla f(x_k)}^2. 
\end{equation}
This requires calculating $ f(x_k + \gamma_k g_k)$ on several candidates stepsizes $\gamma_k$, which is prohibitively expensive since evaluating $f(x)$ requires a full pass over the data.

So instead of using the full function $f(x)$, we can use a stochastic variant where we look for $\gamma_k$ such that
\begin{equation}
 f_{i_k}(x_k + \gamma_k g_k) < f_{i_k}(x_k) -c \gamma_k \norm{\nabla f_{i_k}(x_k)}^2. 
\end{equation}
This is used in the implementation of~\cite{SAG} with $c=\tfrac{1}{2}$ on iterations where $\norm{\nabla f_{i_k}(x_k)}$ is not close to zero. 
 It often works well in practice with appropriate guesses for the trial stepsizes, although no theory exists for the method. 
 %This tends to work well in practice, although theory justifying the algorithm is not yet available.

Alternatively,~\cite{MISO} considers the  ``Bottou trick'' for setting the stepsize in practice.\footnote{Introduced publicly by L\'{e}on Bottou during his tutorial on \SGD methods at NeurIPS 2017.} This method takes a small sample of the dataset (typically 5\%), and performs a binary search that attempts to find the optimal stepsize when performing one pass through this sample. Similar to the Armijo line-search, the stepsize obtained with this method tends to work well in practice but no theory is known for the method.
%\rob{Mark help ! Any reference for this claim? Is it even true?}
%there is no theory to support this. Though there do exist guarantees that \SGD combined with this stochastic line search converges~\citep{Vaswani2019}.

%Due to this expensive  initial computation, variance-reduced methods only reach a higher accuracy solution after a few passed over the data,
%as compared to \SGD  
%
%\subsection{Sparse data implementations}

\subsection{Termination Criteria}

 Iteration complexity results provide theoretical worst-case bounds on the number of iterations to reach a certain accuracy. However, these bounds depend on constants we may not know and in practice the algorithms tend to require fewer iterations than indicated by the bounds. Thus, we should consider tests to decide when the algorithm should be terminated. 

In classic full-gradient descent methods, we typically consider the norm of the gradient $\norm{\nabla f(x_k)}$ or some variation on this quantity to decide when to stop. We can naturally implement these same criteria to decide when to stop an \SVRG method, by using $\norm{\nabla f(\bar{x}_s)}$. For \SAG/\SAGA we do not explicitly compute any full gradients, but the quantity $\bar{g}_k$ converges to $\nabla f(x_k)$ so a reasonable heuristic is to use $\norm{\bar{g}_k}$ in deciding when to stop. In the case of \SDCA, with a small amount of extra bookkeeping it is possible to track the gradient of the dual objective at no additional asymptotic cost. Alternately, a more principled approach is to track the duality gap which adds an $\cO(n)$ cost per iteration but leads to termination criteria with a duality gap certificate of optimality. \revised{An alternative principled approach based on optimality conditions for strongly convex objectives is used in the MISO method~\citep{MISO}, based on a quadratic lower bound~\citep{lin2015universal}.}

\subsection{Reducing Memory Requirement}

Although \SVRG removes the memory requirement of earlier \VR methods, in practice \SAG/\SAGA require fewer iterations than \SVRG on many problems. Thus, we might consider whether there exist problems where \SAG/\SAGA can be implemented with less than $\cO(nd)$ memory. In this section we consider the class of linear models, where the memory requirement can be reduced substantially.

Consider linear models where $f_i(x) = \ell_i(a_i^\top x)$. Differentiating gives
\[
\nabla f_i(x) \;=\; \ell_i'(a_i^\top x)a_i.
\]
Provided we already have access to the feature vectors $a_i$, it is sufficient to store the scalars $\ell_i(a_i^\top x)$ in order to implement the \SAG/\SAGA method. This reduces the memory requirements from $\cO(nd)$ down to $\cO(n)$. \SVRG can also benefit from this structure of the gradients: by storing those $n$ scalars,  we can reduce the number of gradient evaluations required per \SVRG ``inner'' iteration to 1 for this problem class. 

There  exist other problem classes, such as probabilistic graphical models, where it is possible to reduce the memory requirements~\citep{Schmidtnonuni}.

\subsection{Sparse Gradients} 

%\francis{I have put the section on sparse gradients here because it is more natural. May to be shorten a bit}

For problems where the gradients $\nabla f_i(x)$ have many zero values (for example, for linear models with sparse features), the classical \SGD update may be implemented with complexity which is linear in the number of non-zero components in the corresponding gradient, which is often much less than $d$.  This possibility is lost in plain variance reduced methods. However there are two known fixes.

The first one, described by \citet[Section~4.1]{SAG} takes advantage of the simple form of the updates
to implement a ``just-in-time'' variant where the iteration cost
is proportional to the number of non-zeroes. For \SAG (but this applies to all variants), this is done by not 
explicitly storing the full vector  $\displaystyle v^{i_k}$ after each iteration. Instead, in each iteration we
only compute the elements $\displaystyle v^{i_k}_j$
corresponding to non-zero elements, by applying
the sequence of updates to each variable $\displaystyle v^{i_k}_j$
since the last iteration where it was nonzero.

The second one, described by \citet[Section~2]{leblond2017asaga} for \SAGA, adds an extra randomness to the update
$\displaystyle x_{k+1} = x_k - \gamma \left( \nabla f_{i_k}(x_k)-\nabla f_{i_k}(\bar{x}_{i_k})  +\bar{g}_k \right)$, where 
$\nabla f_{i_k}(x_k)$ and $\nabla f_{i_k}(\bar{x}_{i_k})$ are sparse, but $\bar{g}_k $ is dense. The components  of the dense term $(\bar{g}_k)_j $, $j=1,\dots,d$, are replaced by $ w_j (\bar{g}_k)_j$, where $ w \in \mathbb{R}^d$ is a random sparse vector whose support is included in one of the  $\nabla f_{i_k}(x_k)$, and in expectation is the constant vector of all ones. The update remains unbiased (but is now sparse) and the added variance does not impact the convergence rate; the details are given by \citet{leblond2017asaga}.

\section{Advanced Algorithms}
\label{sec:advanced}
%\rob{I places this after the theory section so we can say that accelerated variants converge faster and that proximal methods often don't hurt the convergence.}
%

In this section we consider extensions of the basic \VR methods. Some of these extensions generalize the basic methods to handle more general scenarios, such as problems that are not smooth and/or strongly convex. Other extensions use additional algorithmic tricks or problem structure to design faster algorithms than the basic methods.

\subsection{Hybrid SGD and VR Methods}
\revised{The convergence rate $\rho$ of the \VR methods depends on the the number of training examples $n$.}
% In the linear convergence of \VR methods~\eqref{eq:expconvf}, the convergence rate $\rho$ depends on the condition number $\kappa_{\max}$ as well as the number of training examples $n$.
  This is in contrast to the convergence rates of classic \SGD methods, which are sublinear but do not have a dependence on $n$. This means that \VR methods can perform worse than classic \SGD methods in the early iterations when $n$ is very large. For example, in Figure~\ref{fig:significant} we can see that \SGD is competitive with the two \VR methods throughout the first 10 epochs (passes  over the data). 

Several hybrid \SGD and \VR methods have been proposed to improve the dependence of \VR methods on $n$. \citet{roux2012stochastic} and~\citet{S2GD} analyzed \SAG and \SVRG, respectively, when initialized with $n$ iterations of \SGD. This does not change the convergence rate, but significantly improves the dependence on $n$ in the constant factor. However, this requires setting the stepsize for these initial \SGD iterations, which is more complicated than setting the stepsize for \VR methods\footnote{The implementation of~\citet{SAG} does not use this trick. Instead, it replaces $n$ in Line~7 of Algorithm~\ref{alg:SAG} with the number of training examples that have been sampled at least once. This leads to similar performance and is more difficult to analyze, but avoids needing to tune an additional stepsize.}.
 
More recently, several methods have been explored which guarantee both a linear convergence rate depending on $n$, as well as a sublinear convergence rate that does not depend on $n$. For example,~\citet{lei2017less} show that this ``best of both worlds'' result can be achieved for the ``practical'' \SVRG variant where we use a growing mini-batch approximation of $\nabla f(\bar{x})$.

%Though \SGD does have a slow sublinear rate of convergence, it is unbeatable in the first pass over the data. In fact in Figure~\ref{fig:significant} we can see that \SGD is competitive with the two variance-reduced methods on the first 10 epochs (passes  over the data). This is explained in part due to the initialization costs of \VR methods, such as allocating the memory for the auxiliary $v^i$ vectors.
%To initialize most variance-reduced methods we need to compute $\nabla f_i(w_0)$ for $i=1,\ldots, n.$
% This initialization cost only pays off after a few passes over the data. 

%To get the best of both worlds, we can use a hybrid of \SGD and the \VR methods. For the first pass over the data we use \SGD before passing to a \VR method. Some \VR methods like \SAG/SAGA can then make use of the stochastic gradient computed in the first pass to initialize the $v^i$ vectors.

\subsection{Non-Uniform Sampling}
\label{sec:NUS}

Instead of improving the dependence on $n$, a variety of works have focused on improving the dependence on the Lipschitz constants $L_i$ by using non-uniform sampling of the random training example $i_k$. In particular, these algorithms bias the choice of $i_k$ towards the larger $L_i$ values. This means that examples whose gradients can change more quickly are sampled more often. \revised{This is typically combined with using a larger stepsize that depends on the average of the $L_i$ values rather than the maximum $L_i$ value.
Under an appropriate choice of the sampling probabilities and stepsize}, this leads to improved iteration complexities of the form
\[
\cO((\kappa_{\text{mean}} + n)\log(1/\varepsilon)),
\]
which depends on $\kappa_{\text{mean}} \eqdef 
( \tfrac{1}{n}\sum_i L_i )/\mu$ rather than on $\kappa_{\max} = (\max_i L_i)/\mu$. 
This improved rate under non-uniform sampling has been shown for the basic \VR methods \SVRG~\citep{proxSVRG}, \SDCA~\citep{Qu2015b}, and \SAGA~\citep{Schmidtnonuni,GowerRichBach2018}. \peter{Add citation to  SVRG Qu-Qian.}

Virtually all existing methods use fixed probability distribution over  $\{1,\dots,n\}$ throughout the iterative process. However, it is possible to further improve on this choice by adaptively changing the probabilities during the execution of the algorithm. The first \VR method of this type, \ASDCA, was developed by  \citep{ADASDCA} and is based on updating the probabilities in \SDCA by using what is known as the {\em dual residue}.

~\citet{SAG} present an empirical method that tries to estimate local values of the $L_i$ (which may be arbitrarily smaller than the global values), and show impressive gains in experiments. A related method with a theoretical backing that uses local $L_i$ estimates is presented by~\citet{vainsencher2015local}.

\subsection{Mini-batching}

%\schmidt{Old text from old Section 3 on classic \SGD methods}
%One natural way to reduced the variance in \SGD is to use a mini batch of stochastic gradients to get a better estimate of the full gradient. Let $B \subset \{1,\ldots, n\}$ be a set, which we call a mini batch, where $|B| =b \in \N.$
%Mini batch \SGD samples uniformly at average from all sets of size $b$ a mini-batch $B_k$ and iterates 
%\begin{equation}\label{eq:grad}
%x_{k+1} = x_k -  \gamma\frac{1}{b} \sum_{i\in B} \nabla f_i(x_k).
%\end{equation} 
%Using a larger mini batch decreases the variance, all the way to zero variance when $b=n$ and \SGD becomes full gradient descent.   But excluding the extreme $b=n,$ \SGD with with mini-batching still converges sublinearly~\citep{Shalev-Shwartz2011,Gower19}.
%\schmidt{End of old text}

%The \VR methods can be used with mini batching by using  the mini batch estimate $\frac{1}{b}\sum_{j\in B} \nabla f_j(x)$ in the place of using a single stochastic gradient $\nabla f_i(x),$ where $B\subset \{1,\ldots, n\}$ with $|B| =b.$ But what mini-batch size $b$ should we use? Furthermore, how should the stepsize $\gamma$ change as a function of the mini batch size $b$? 

Another strategy to improve the dependence on the $L_i$ values is to use mini-batching, analogous to the classic mini-batch \SGD method~\eqref{eq:mini}, to obtain a better approximation of the gradient.  \revised{There are a number of ways of doing mini-batching, but here we will focus on a fixed batch-size chosen uniformly at random. That is, let $b \in \N$ and we choose a set $B_k \subset \{1,\ldots, n\}$ with  $|B_k|=b$ with uniform probability from all sets  with $b$ elements. 
We can now implement the \VR method by replacing each $\nabla f_i(x^k)$ with a mini-batch estimate given by $\frac{1}{b} \sum_{i \in B_k}\nabla f_i(x^k)$.}

There were a variety of early mini-batch methods~\citep{pegasos2,Konecny2015a,Konecny:wastinggrad:2015,HofmannLLM152015}, but the most recent methods are able to obtain an iteration complexity of the form~\citep{Qu2015b,SAGA-AS,SAGAminib,SVRGminib}
\begin{equation}\label{eq:complexityminib}
\cO \left(\frac{\cL(b)}{\mu}+\frac{n}{b} \right)\log\left( \frac{1}{\varepsilon}\right),
\end{equation} 
using a stepsize of $\gamma = \cO(1/\cL(b))$, where 
\begin{equation} \label{eq:expsmooth}
\cL(b) \; =\; \frac{1}{b}\frac{n-b}{n-1}L_{\max} + \frac{n}{b}\frac{b-1}{n-1}L,
\end{equation}
 is a \emph{mini batch} smoothness constant first defined in~\citep{GowerRichBach2018,Gower19}.
This  iteration complexity interpolate between the complexity of full-gradient descent when $\cL(n) = L$ and \VR methods where $\cL(1) = L_{\max}$. Since  $L\leq L_{\max} \leq n L$  it  is possible that $L \ll L_{\max}$. Thus using larger mini batches allows for the possibility of large speedups, especially in settings where we can use parallel computing to evaluate multiple gradients simultaneously. However, computing $L$ is typically more challenging than computing $L_{\max}$.

 \revised{For generalized linear models~\eqref{eq:linear_optimization} with $\ell'' <M$ we have that $L \leq M \lambda_{\max}(\mA\mA^\top) $ where $\mA = [a_1,\ldots, a_n] \in \R^{d\times n}$ and $\lambda_{\max}(\cdot)$ is the largest eigenvalue function. The largest eigenvalue can be computed using a reduced SVD at a cost of $\cO(d^2 n)$ or by using a few iterations of the power method to get an approximation. Unfortunately, for some problems this cost may negate the advantage of using mini batches. In such a case, we could replace $L$ with the upper bound $\frac{1}{n} \sum_{i=1}^n L_i.$ But this may be a very conservative upper bound.}

\subsection{Accelerated Variants}
\peter{Not sure about~\citep{shalev2014accelerated} being the first. Mention L-Katyusha as well.}
An alternative strategy for improving the dependence on $\kappa_{\max}$ is Nesterov or Polyak \emph{acceleration} (aka {\em momentum}). It is well-known that Nesterov's accelerated \GD improves the iteration complexity of the full-gradient method from $\cO(\kappa \log(1/\varepsilon))$ to $\cO(\sqrt{\kappa}\log(1/\varepsilon))$~\citep{nesterov1983method}. Although we might naively hope to see the same improvement for \VR methods, replacing the $\kappa_{\max}$ dependence with $\sqrt{\kappa_{\max}}$, we now know that the best complexity we can hope to achieve is $\cO((\sqrt{n\kappa_{\max}} + n)\log(1/\varepsilon))$~\revised{\citep{nattylower2016,lan2018optimal}, which was first achieved by an accelerated SDCA method~\citep{shalev2014accelerated}}. Nevertheless, this complexity still guarantees better worst-case performance in ill-conditioned settings (where $\kappa_{\max} \gg n$).
  
A variety of \VR methods have been proposed that incorporate an acceleration step in order to achieve this improved complexity~\citep[see, e.g.,][]{shalev2014accelerated,zhang2017stochastic,allen2017katyusha}. 
Moreover, the ``catalyst'' framework of~\citet{lin2015universal} can be used to modify \emph{any}  method achieving the complexity $\cO((\kappa_{\max} + n)\log(1/\varepsilon))$ to an accelerated method with a complexity of $\cO((\sqrt{n\kappa_{\max}} + n)\log(1/\varepsilon))$.

% Variance-reduced techniques require a number of iterations proportional to $\kappa_{\max} +n$, where $\kappa_{\max}$ is the condition number and $n$ the number of functions. For ill-conditioned problems where $\kappa_{\max} $ is significantly larger than $n$, simple accelerated variants require only a number of iterations proportional to  $\sqrt{\kappa_{\max}n} +n$. This can be achieved  using the ``catalyst'' framework of~\citet{lin2015universal} which solves a sequence of optimization problems each with a gradually increase the condition number, or directly by a ``Nesterov extrapolation step'' similar to momentum on most of the variants we have described~\citep[see, e.g.,][]{shalev2014accelerated,allen2017katyusha}. Note that such accelerated techniques have the optimal running-time complexity among methods that access individual gradient of functions and combine them linearly~\citep{lan2018optimal}. 
 %Point to Julien Maihral's recent software package showing how accelerated methods can be implemented to give real improvements in time.

%\subsection{Distributed Variants}
%
%\peter{Non-VR version \citep{alistarh2017qsgd}; VR versions: SEGA \citep{SEGA}, DIANA \citep{DIANA, DIANA2}, ICD \citep{99percent}}
%\peter{I will write this.}

\subsection{Relaxing Smoothness}

A variety of methods have been proposed that relax the assumption that $f$ is $L$-smooth. The first of these was the \SDCA method, which can still be applied to~\eqref{eq:optimization_problem} even if the functions $\{f_i\}$ are non-smooth. This is because the dual remains a smooth problem. A classic example is the support vector machine (SVM) loss, where $f_i(x) = \max\{0,1-b_ia_i^\top x\}$. This leads to a convergence rate of the form $\cO(1/\varepsilon)$ rather than $\cO(\log(1/\varepsilon))$, so does not give a worst-case advantage over classic \SGD methods \peter{Pegassus 2}. However, unlike classic \SGD methods, we can optimally set the stepsize when using \SDCA. Indeed, prior to new wave of \VR methods, dual coordinate ascent methods have been among the most popular approaches for solving SVM problems for many years. For example, the widely-popular libSVM package~\citep{Chang2011} uses a dual coordinate ascent method.

One of the first ways to handle non-smooth problems that preserves the linear convergence rate was through the use of \emph{proximal-gradient} methods. These methods apply when $f$ has the form $f(x) = \frac 1 n\sum_{i=1}^n f_i(x) + \Omega(x)$. In this framework it is assumed that $f$ is $L$-smooth and that the regularizer $\Omega$ is convex on its domain. But the function $\Omega$ may be non-smooth and may enforce constraints on $x$. However, $\Omega$ must be ``simple'' in the sense that it is possible to efficiently compute its proximal operator applied to \revised{step of \GD, that is
\begin{equation}\label{eq:proxgrad}
 x_{k+1} = \argmin_{x\in\R^d}  \frac{1}{2} \norm{ x - (x_k - \gamma \nabla f(x_k))  }^2 + \gamma \Omega(x),
\end{equation}
should be relatively inexpensive to compute.
The above method is known as the proximal-gradient algorithm~\citep[see, e.g.,][]{combettes2011proximal}, and it achieves the iteration complexity $\cO(\kappa\log(1/\varepsilon))$ even though $\Omega$ (and consequently $f$) is not $L$-smooth or even necessarily differentiable.
A common example where~\eqref{eq:proxgrad} can be efficiently computed is the L1-regularizer}, where $\Omega(x) = \lambda\norm{x}_1$ for some regularization parameter $\lambda > 0$.

A variety of works have shown analogous results for proximal variants of \VR methods. These works essentially replace the iterate update by an update of the above form, with $\nabla f(x_k)$ replaced by the relevant approximation $g_k$. This leads to iteration complexities of $\cO((\kappa_{\max} + n)\log(1/\varepsilon))$ for proximal variants of \SAG/\SAGA/\SVRG if the functions $f_i$ are $L_i$ smooth~\citep{proxSVRG,SAGA}. 
  
Several authors have also explored combinations of \VR methods with the alternating direction method of multipliers (\ADMM) approaches~\citep{zhong2013fastADMM}, which can achieve improved rates in some cases where	 $\Omega$ does not admit an efficient proximal operator. Several authors have also considered the case where the individual $f_i$ may be non-smooth, replacing them with a smooth approximation~\citep[see][]{allen2017katyusha}. This smoothing approach \peter{Link to Nesterov} does not lead to linear convergence rates, but leads to faster sublinear rates than are obtained with \SGD methods on non-smooth problems (even with smoothing).

\subsection{Relaxing Strong Convexity}
\label{sec:relaxstrconv}
%\rob{Needs to be totally re-written}
%\schmidt{convex, PL, PCA, non-convex (SARAH, etc.)}

While we have focused on the case where $f$ is strongly convex and each $f_i$ is convex,  these assumptions can be relaxed. For example, if $f$ is convex but not strongly convex then early works showed that \VR methods achieve a convergence rate of $\cO(1/k)$~\citep{SAG,MISO,S2GD,mahdavi2013mixedgrad,SAGA}. This is the same rate achieved by \GD under these assumptions, and is faster than the $\cO(1/\sqrt{k})$ rate of \SGD in this setting.

Alternatively, more recent works replace strong convexity with weakened assumptions like the PL-inequality and KL-inequality~\citep{PL1963}, that include standard problems like (unregularized) least squares but where it is still possible to show linear convergence
~\citep{gong2014linear,karimi2016linear,reddi2016fast,pmlr-v48-reddi16}. 
%Such works have led to faster algorithms for the classic principal component analysis problem.
\peter{Miso -> Qian.  L-SVRG -> (Qu, Qian, Richt)}

\subsection{Non-convex Problems}
\label{sec:nonconvex}
%
%By gradually relaxing our convexity assumptions,...

%Non-convex problems can also benefit from variance reduction techniques. 
\revised{Since 2014 a sequence of papers have gradually relaxed the convexity assumptions on the functions $f_i$ and $f$, and adapted the variance reduction methods to achieve state-of-the-art complexity results for several different non-convex settings. Here we summarize some of these results, starting with the setting closest to the strongly convex setting and gradually relaxing any such convexity assumptions.
  For a more detailed discussion see~\citet{zhou2019lower} and~\citet{SPIDER}.}
%
%
% there has been a push to relax the assumption that either the $f_i$ or $f$ are convex, and to understand how to adapt the variance reduced methods 
%
%Variance reduction techniques have been used to achieve state-of-the-art complexity results for a number of different non-convex settings. 

%Recently that has been a lot of progress in developing \VR method for minimizing non-convex functions.
\revised{
The first assumption we relax is the convexity of the $f_i$ functions. % but their average $f$ is still strongly convex.
The first work in this direction was for solving the PCA problem where the individual $f_i$ functions are non-convex~\citep{Shamir2014,garber2015b}.
Both~\citet{garber2015b} and~\citet{dfSDCA} then showed how to use the catalyst  framework~\citep{lin2015universal} to devise algorithms with an iteration complexity of  $\cO(n+n^{3/4}\sqrt{L_{\max}}/\sqrt{\mu})\log(1/\varepsilon)$ for $L_i$--smooth non-convex $f_i$ functions so long as their average $f$ is  $\mu$--strongly convex. Recently, this complexity was shown to match the lower bound in this setting~\citep{zhou2019lower}. 
%Thus this iteration complexity cannot be improved under these assumptions and is optimal in this sense.
\citet{pmlr-v70-allen-zhu17a}  took a step further and relaxed the assumption that $f$ is strongly convex, and instead, allowed $f$ to be simply convex or even have ``bounded non-convexity'' (strong convexity but with a negative parameter $-\mu$).  To tackle this setting, \citet{pmlr-v70-allen-zhu17a}   proposed an accelerated variant of \SVRG  that achieves state-of-the-art complexity results which recently have been shown to be optimal~\citep{zhou2019lower}. \footnote{Note that, when assuming that $f$ has a bounded non-convexity,  the complexity results are with respect to finding a point $x_k\in\R^d$ such that $\E{\norm{\nabla f(x_k)}^2} \leq \varepsilon$. }
}

\revised{
Even more recently, the convexity assumptions on $f$ have been completely dropped. 
By only assuming that the $f_i$'s are smooth and $f$ has a lower bound,~\citet{SPIDER} present an algorithm based on the continuous update~\eqref{eq:SARAH} that finds an approximate stationary point  such that $\E{\norm{\nabla f(x)}^2} \leq \varepsilon$ using $\cO\left( n + \revised{\sqrt{n}}/\varepsilon^2\right)$  iterations. Concurrently to this,~\citet{SNSVRG2018} presented a more involved variant of \SVRG that uses  multiple  reference points and achieves the same iteration complexity. \citet{SPIDER} also provided a lower bound showing that  the preceding complexity 
is optimal under these assumptions.}

\revised{An interesting source of non-convex functions are problems where the objective is a composition of functions $f(g(x))$, where $g: \R^d \rightarrow \R^m$ is a mapping. Even when both $f$ and $g$ are convex, their composition may be non-convex. There are several interesting applications where either $g$ is itself an average (or even expectation) of maps, or $f$ is an average of functions, or even both~\citep{lian2016svrg}. In this setting, the finite sum structures can also be exploited to develop variance reduced methods, most of which are based on variants of \SVRG. In the setting where $g$ is a finite sum, state-of-the-art complexity results have been achieved using the continuous update~\eqref{eq:SARAH}, see~\citep{zhang2020stochastic}.}
%\citep{pmlr-v97-horvath19a}

%In some cases, the strong convexity assumption can be relaxed to the PL (Polyak-Łojasiewicz) condition, and the \VR methods still enjoy linear convergence.
%Otherwise, if $f(x)$ is simply convex, most \VR methods have been shown to have a sublinear iteration complexity of $\cO(1/\varepsilon)$. If we further relax the convexity assumption, the \VR methods can still have favourable convergence properties.

%If $f(x)$ is not convex, then the \VR methods have a
%Get $\cO(1/\varepsilon)$ for convex.
%Get $\cO(\log(1/\varepsilon)$ for PL which covers least but is strong.
%For instance, for solving the PCA (Principal Components Analysis) problem, which is non-convex, several methods exist that get faster rates for non-convex.

%For more general nonconvex problems, the SARAH method has an iteration complexity of  $O\left( n + 1/\varepsilon^2\right)$  for finding a point such that $\E{\norm{\nabla f(x)}^2} \leq \varepsilon,$ which is the best known iteration complexity for a \VR methods in the nonconvex setting~\citep{pmlr-v97-horvath19a}

%\francis{SARAH is not the only non-convex variance reduction technique. Sra and Reddi have one. We cannot cite one without the other}

\subsection{Second-Order Variants}

Inspired by Newton's method, there are now second-order variants of the \VR methods of the form
\begin{equation}\label{eq:secondorder}
x_{k+1} =  x_k -\gamma_k \mH_k g_k,
\end{equation}
where $\mH_k \in \R^{d\times d}$ is an estimate of the inverse Hessian matrix $\nabla^2 f(x_k)^{-1}$. 
The challenge in designing such methods is finding an efficient way to update $\mH_k$ that results in a sufficiently accurate estimate and does not cost too much. This is a difficult balance to achieve, for if $\mH_k$ is a poor estimate it may do more damage to the convergence of~\eqref{eq:secondorder} than help. On the other hand, expensive routines for updating $\mH_k$ can make the method inapplicable  when the number of data points is large.
%This $\mH_k$ estimate needs to be updated in efficiently so as to not offset the carefully.

%\schmidt{I think this section needs work. It should spend less time on the deterministic and classic SGD variants, and more time on the VR variants.}
%\peter{I agree. I am yet to take a pass.}

Though difficult, the rewards are potentially high. Second order methods can be invariant (or at least insensitive) to coordinate transformations and ill-conditioned problems. This is in contrast to the first order methods that often require some feature engineering, preconditioning and data preprocessing to converge.

Most of the second order \VR methods are based on the \BFGS quasi-Newton update~\citep{Broyden1967,Fletcher1970,Goldfarb1970,Shanno1971}.

The first stochastic \BFGS method that makes use of subsampling was the online \LBFGS method~\citep{Schraudolph2007} which uses subsampled gradients to approximate Hessian-vector products. The regularized \BFGS method~\citep{Mokhtari2013,Mokhtari2014} also makes use of stochastic gradients, and further modifies the \BFGS update by adding a regularizer to the $\mH_k$ matrix.

The first method to use subsampled Hessian-vector products in the \BFGS update, as opposed to using differences of stochastic gradients, was the \SQN method~\citep{Byrd2015}.
 \cite{moritz2016linearly} then proposed combining  \SQN with  \SVRG. The resulting method performs very well in numerical tests and was the first example of a 2nd order \VR method with a proven linear convergence, though with a significantly worse complexity than the $\mathcal{O}(\kappa_{\max} +n)\log(1/\epsilon)$ rate of the \VR methods.
%\peter{Check the above statement.} 
 \cite{moritz2016linearly}'s method was later extended to a block quasi-Newton variant and analysed with an improved complexity by~\cite{GowerGold2016}.
There are also specialized variants for the non-convex setting~\citep{Wang2014}.
\revised{These 2nd order quasi-Newton variants of the \VR methods are hard to analyse  and,  as far as we are aware, there exists no quasi-Newton variants of~\eqref{eq:secondorder} that have an update cost independent of $n$ and is proven to have a better global complexity than the $\cO((\kappa_{\max} + n)\log(1/\varepsilon))$ of \VR methods. }

\revised{However, there do exist stochastic Newton type methods, such as \emph{Stochastic Dual Newton Ascent} ({\tt SDNA})~\citep{SDNA} that perform minibatch type Newton steps in the dual space, with a cost that is independent of $n$ and does have a better convergence rate than \SDCA. 
%\peter{This is not true. For instance, {\tt SDNA} of \cite{SDNA} is better in theory than  \SDCA.} 
Furthermore, more recently  \cite{SN2019} preposed a minibatch Newton method with a {\em local} linear convergence rate of the form $\cO((n/b)\log (1/\varepsilon))$ that is \emph{independent of the condition number}, where $b$ is the mini-batch size. } 

\revised{Yet another line of stochastic second order methods is being developed by combining variance reduction techniques with the cubic regularized Newton method~\citep{nesterov2006cubic} . These methods replace both the Hessian and gradient by variance reduced estimates and then minimize at each iteration an approximation of the  second order Taylor expansion with an additional cubic regularizer. What is particularly promising about these methods is that they achieve state-of-the-art sample complexity, in terms of accesses gradients and Hessians of individual functions $f_i$, for finding a second-order stationary points for smooth non-convex problems~\cite{JMLR:v20:19-055,Lansamplecubic2018}.
This is currently a very active direction of research.}
\peter{Mention Kovlev. Mention Roosta-Mahoney MAPR~\cite{Roosta-Khorasani2016}}
\rob{Note that~\cite{Roosta-Khorasani2016} uses full gradient evaluations. Does this really enter the category of being about variance reduction?}
%Yet another class of second order methods use variance-reduced gradient and Hessian estimates to find second-order stationary points for non-convex problems. 
%Another line of second methods that use variance-reduced gradient and Hessian estimates to find second-order stationary points for non-convex problems. 
%% in Cubic regularization [*5]. 
%
%[*6] introduces a Stochastic Variance-Reduced Cubic Regularized Newton Methods (SVRC) which uses variance reduced gradient from [*7, *8] and Hessian estimators based on variance reduction with a cubic regularization method. For functions with $\rho$--Lipschitz Hessian matrices, SVRC finds $(\epsilon, \sqrt{\epsilon})$-second-order stationary points within $\cO(n+n^{4/5}/\epsilon^{3/2})$ second-order oracle calls and $\cO(n+n^{2/3}/\epsilon^{3/2})$ stochastic Hessian computations.

%\section{What is left to do?}
%\rob{Stochastic line searches that provenly work? Fill the space between \VR and \SGD, that is, what is between exponential convergence and sublinear? This could lead to a trade-off between higher variance with less time and memory footprint of the method. Questions in the non-convex setting?}

\section{Discussion and Limitations}

%\schmidt{Maybe mention some limitations: only for the training error, doesn't work for neural networks, some methods are more-complicated to implement, maybe not-so-easy to parallelize, others?}

In the previous section, we presented a variety of extensions of the basic \VR methods. We presented these as separate extensions, but many of them can be combined. For example, we can have an algorithm that uses mini-batching and acceleration while using a proximal-gradient framework to address non-smooth problems. The literature has now filled out most of these combinations.

Note that classic \SGD methods apply to the general problem of minimizing a function of the form $f(x) = \mathbb{E}_z f(x,z)$ for some random variable $z$. In this work we have focused on the case of the training error, where $z$ can only take $n$ values. However, in machine learning we are ultimately interested in the \emph{test error} where $z$ may come from a continuous distribution. If we have an infinite source of examples, we can use them within \SGD to directly optimize the test loss function. Alternatively, we can treat our $n$ training examples as samples from the test distribution, then doing 1 pass of \SGD through the training examples can be viewed as directly making progress on the test error. Although it is not possible to improve on the test error convergence rate by using \VR methods, several works show that \VR methods can improve the constants in the test error convergence rate compared to \SGD~\citep{frostig2015competing,Konecny:wastinggrad:2015}.

We have largely focused on the application of \VR methods to linear models, while mentioning several other important machine learning problems such as graphical models and principal component analysis. One of the most important applications in machine learning where \VR methods have had little impact is in training deep neural networks. Indeed, recent work shows that \VR may be ineffective at speeding up the training of deep neural networks~\citep{adefazio-varred2019}. On the other hand, \VR methods are now finding applications in a variety of other machine learning applications including policy evaluation for reinforcement learning~\citep{NIPS2019_8814,pmlr-v80-papini18a,XuGG19}, expectation maximization~\citep{NIPS2018_8021}, simulations using Monte Carlo methods~\citep{NIPS2019_8639}, saddle-point problems~\citep{palaniappan2016stochastic}, and generative adversarial networks~\citep{chavdarova2019reducing}.

%Note that \SGD applies more generally to any function of the form $f(x) = \mathbb{E}_z f(x,z)$, where $z$ is any random variable, potentially taking more than $n$ values. Thus, when only a single pass is made and the data is assumed to be sampled independently from the same distribution, \SGD leads to guarantees directly on the generalization performance (which is not computable given only observations). However, in practice we typically obtain better performance by using regularization and multiple passes through a dataset.
%\francis{I have made an attempt at taling about test error.}
%\schmidt{I moved the test error stuff here, but it is still a bit awkward in that it interrupts the story. Can we move this to the discussion at the end, or put this in the variants section?}

\section*{Acknowledgement}
We thank Quanquan Gu, Julien Mairal, Tong Zhang and Lin Xiao for valuable suggestions and comments on an earlier draft of this paper. In particular, Quanquan's recommendations for the non-convex section improved the organization of our Section~\ref{sec:nonconvex}.

\bibliography{VR_review}

\appendix

\section{Lemmas}
Here we provide and prove several auxiliary lemmas.

\begin{lemma}\label{lem:expsmooth}
Let $f_i(x)$ be convex and $L_{\max}$-smooth for $i=1,\ldots, n.$ Let $i$ be sampled uniformly on average from $\{1,\ldots, n\}$. It follows that for every $x \in \R^d$ we have that
\begin{equation} \label{eq:expsmooth}
\EE{k}{\norm{\nabla f_i(x) -\nabla f_i(x_\star)}^2} \; \leq \; 2L_{\max} (f(x) - f(x_\star)).
\end{equation}
\end{lemma}
\begin{proof}
Since $f_i$ is convex we have that
\begin{equation}\label{eq:conv}
f_i(z) \geq f_i(x_\star) + \dotprod{\nabla f_i(x_\star), z-x_\star}, \quad \forall z \in \R^d,
\end{equation}
and since $f_i$ is smooth we have that
\begin{equation}\label{eq:smoothnessfunc}
f_i(z) \leq f_i(x) + \dotprod{\nabla f_i(x), z-x} + \frac{L_{\max}}{2} \norm{z-x}^2, \quad \forall z,x \in \R^d,
\end{equation}
see Section 2.1.1 in~\cite{NesterovBook} for more equivalent ways of re-writing convexity and smoothness.
It follows for all $x,z \in \R^d$ that
\begin{eqnarray*}
f_i(x_\star) -f_i(x) & = & f_i(x_\star) -f_i(z)+f_i(z) - f_i(x)\\ &\overset{\eqref{eq:conv}+\eqref{eq:smoothnessfunc} } \leq &
\dotprod{\nabla f_i(x_\star), x_\star-z} + \dotprod{\nabla f_i(x), z-x}  \\
& & \quad +\frac{L_{\max}}{2}\norm{z-x}^2.
\end{eqnarray*}
To get the tightest upper bound on the right hand side, we can minimize the right hand side in $z$, which gives
\begin{equation}
z = x - \frac{1}{L_{\max}}(\nabla f_i(x) -\nabla f_i(x_\star)).
\end{equation}
Substituting this in the above gives
\begin{align}
f_i(x_\star) -f_i(x) & = 
\dotprod{\nabla f_i(x_\star), x_\star-x + \frac{1}{L_{\max}}(\nabla f_i(x) -\nabla f_i(x_\star))} \nonumber \\
&  \quad - \frac{1}{L_{\max}}\dotprod{\nabla f_i(x), \nabla f_i(x) -\nabla f_i(x_\star)}  \nonumber \\
&  \quad+\frac{1}{2L_{\max}}\norm{\nabla f_i(x) -\nabla f_i(x_\star)}^2 \nonumber \\
%& =\dotprod{\nabla f(y), y-x}   - \frac{1}{L}\norm{\nabla f(x)-\nabla f(y)}^2 +\frac{1}{2L}\norm{\nabla f(x) -\nabla f(y)}^2\\
&=  \dotprod{\nabla f_i(x_\star), x_\star-x}  \nonumber \\
& \quad - \frac{1}{2L_{\max}}\norm{\nabla f_i(x)-\nabla f_i(x_\star)}^2. \nonumber
\end{align}
Taking expectation over $i$ in the above and using that $\EE{i}{f_i(x)} = f(x)$ and $\E{\nabla f_i(x_\star) } =0,$ gives the result.
\end{proof}

\begin{lemma}\label{lem:varbndX}
Let $X \in \R^d$ be a random vector with finite variance. It follows that
\begin{equation}
\E{\norm{X-\E{X}}^2} \; \leq \; \E{\norm{X}^2}.
\end{equation}
\end{lemma}
\begin{proof}
\begin{align}
\E{\norm{X-\E{X}}^2} &= \E{\norm{X}^2} -2\E{\norm{X}}^2 + \E{\norm{X}}^2 \nonumber \\
& = \E{\norm{X}^2} -\E{\norm{X}}^2 \leq \E{\norm{X}^2}.
\end{align}
\end{proof}

%
%\section{The Theory of \SVRG}
%\label{sec:Asvrg}
%
%
%\SVRG was first present and show to converge by
%\cite{Johnson2013} where they showed that by setting the parameters 
%\[(\gamma, \; m) \;=\; \Big( \frac{1}{10 L_{\max}}, \left\lceil \frac{L_{\max}}{\mu} \right\rceil \Big),\] the reference point as the average of the inner iterates, that is by using $p_k = \frac{1}{m}$ on line~\ref{ln:theory} so that 
%$\bar{x}_s = \frac{1}{m}\sum_{k=1} x_k$ and in addition by resetting the inner iterates to the new reference point with $x_s^0 = \bar{x}_s$ (as opposed to $x_0 = x_m$ as given on line~\ref{ln:inneriterreset}) then under Assumption~\ref{ass:strconvsmooth} the \SVRG algorithm converges at a rate of $\rho = 0.1$. See Theorem 6.5 by~\cite{Bubeck:2015} for a clear proof of this statement. 
%
%
%
%
%An analysis that is closer to what is done...
%
%\begin{equation}
%    S_m \eqdef \sum_{i=0}^{m-1}(1-\gamma\mu)^{m-1-i} \quad \mbox{and} \quad p_t \eqdef  \frac{(1 - \gamma\mu)^{m-1-t}}{S_m}, \label{eq:Smpts}
%\end{equation}
%for $t=0,\ldots, m-1$.
%

\section{Convergence Proof Illustrated via \SGDstar}\label{sec:theory}
\label{sec:proof}

%\schmidt{Need to change this to use function values, and give the final iteration complexity.}
%
%The variance-reduced methods enjoy a fast linear convergence. By examining the proof of the this convergence, we can also see how the issue of variance in the gradient estimator $g_k$ arises....
For all \VR methods, the first steps of proving convergence are the same. First we expand 
\begin{align*}
\norm{x_{k+1}-x_\star}^2 & \; =\; \norm{x_k -x_\star  - \gamma g_k}^2  \\
& \; =\;\norm{x_k  -x_\star}^2 -2\gamma \dotprod{x_k -x_\star, g_k}  + \gamma^2 \norm{g_k}^2.
\end{align*}
Now taking expectation conditioned on $x_k$ and using~\eqref{eq:sgdunbiased}, we arrive at 
\begin{align*}
\EE{k}{\norm{x_{k+1}-x_\star}^2} 
& \; =\;\norm{x_k  -x_\star}^2+ \gamma^2 \EE{k}{\norm{g_k}^2 }  \\
& \quad -2\gamma \dotprod{x_k -x_\star, \nabla f(x_k)} .
\end{align*}
Using either convexity or strong convexity, we can get rid of the $\dotprod{x_k -x_\star, \nabla f(x_k)} $ term. In particular, since $f(x)$ is  $\mu$-strongly convex, we have  
\begin{align}
\EE{k}{\norm{x_{k+1}-x_\star}^2 } 
& \; \leq \; (1-\mu \gamma )\norm{x_k  -x_\star}^2 + \gamma^2 \EE{k}{\norm{g_k}^2 } \nonumber \\
& \quad -2\gamma (f(x_k) -f(x_\star)). \label{eq:a3a3daa3}
\end{align}
To conclude the proof, we need a bound on the second moment $\EE{k}{\norm{g_k}^2 }$ of $g_k$. For the plain vanilla \SGD, often it is simply assumed that this variance term is bounded uniformly by an unknown constant $B>0$. But this assumption rarely holds in practice, and even when it does, the resulting convergence speed depends on this unknown constant $B$.
In contrast, for a \VR method we can explicitly control the second moment of $g_k$ since we can control the variance of $g_k$ and
\begin{equation}
\EE{k}{\norm{g_k -\nabla f(x_k)}^2}  \; = \; \E{\norm{g_k}^2} - \norm{\nabla f(x_k)}^2. 
\end{equation}
%We could fix this issue if we knew $\nabla f_i(x_\star)$, that is the $f_i$ gradients at the optimal point. Since in this case we could use 
%\begin{equation}\label{eq:gkstar}
% g_k = \nabla f_i(x_k) -\nabla f_i(x_\star)
%\end{equation} 
% with $i$ sampled uniformly on average. 

To illustrate, we now prove the convergence of \SGDstar.
\begin{theorybox}
\begin{theorem}[\cite{Gorbunov20unified}]
Consider the iterates of \SGDstar~\eqref{eq:sgdstar}. If Assumptions~\ref{ass:lipschittz} and~\ref{ass:SC} hold and  $\gamma \leq \frac{1}{L_{\max}},$ then the iterates converge linearly with
\begin{equation}
\E{\norm{x_{k+1}-x_\star}^2 } \; \leq \; (1- \gamma \mu )\E{\norm{x_k  -x_\star}^2}.
\end{equation}
Thus,  the iteration complexity of \SGDstar is given by
\begin{equation} \label{eq:complexunistar}
  k \; \geq \; \frac{L_{\max}}{\mu}\log\left(\frac{1}{\varepsilon}\right) \;\; \Rightarrow \;\; \frac{\E{\norm{x_{k}-x_\star}^2}}{\norm{x_0 -x_\star}^2}  < \varepsilon.
\end{equation}
\end{theorem}
\end{theorybox}
\begin{proof}
Using Lemma~\ref{lem:expsmooth},  we have 
\begin{align} 
\EE{k}{\norm{g_k}^2} &= 
\EE{k}{\norm{\nabla f_i(x_k) -\nabla f_i(x_\star)}^2} \nonumber \\
& \leq  2L_{\max} (f(x_k) - f(x_\star)).\label{eq:expsmooth}
\end{align}
Using  the above in~\eqref{eq:a3a3daa3} we have
\begin{align}
\EE{k}{\norm{x_{k+1}-x_\star}^2 } 
& \; \leq \; (1-\mu \gamma )\norm{x_k  -x_\star}^2 \nonumber\\
& \quad + 2\gamma (\gamma L_{\max} -1 ) (f(x_k) -f(x_\star)). \label{eq:aiansidnliasn}
\end{align}

Now by choosing $\gamma \leq \frac{1}{L_{\max}}$ we have that
$\gamma L_{\max} -1 <0$ and consequently $2\gamma (\gamma L_{\max} -1 ) (f(x_k) -f(x_\star))$ is negative since  $f(x_k) -f(x_\star) \geq 0$. So it now follows by taking  expectation in~\eqref{eq:aiansidnliasn} that
\[
\E{\norm{x_{k+1}-x_\star}^2 } \; \leq \; (1-\mu \gamma )\E{\norm{x_k  -x_\star}^2}.
\]
\end{proof}
%This is a very strong form of convergence of stochastic methods known as linear convergence because on a log plot, the error $\norm{x_{k+1}-x_\star}^2$ decreases linearly with the number of iterations (see Significance Statement). 

This proof also shows that the shifted \SGD method is a variance-reduced method. Indeed, since

\begin{align*}
\E{\norm{g_k - \nabla f(x_k)}^2} & = \E{\norm{\nabla f_i(x_k) -\nabla f_i(x_\star) - \nabla f(x_k)}^2} \\
%&=  \E{\norm{g_k +\nabla f(x_\star) - \nabla f(x_k)}^2} \\
&\leq \E{\norm{\nabla f_i(x_k) -\nabla f_i(x_\star)}^2} \\
& \leq 2L_{\max} (f(x_k) - f(x_\star)),
\end{align*} 
where in the first inequality we used  Lemma~\ref{lem:varbndX} with $X = \nabla f_i(x_k) -\nabla f_i(x_\star)$.

%\section*{Biography}

%\rob{Need to add our bibliography with photo!}
%\begin{IEEEbiography}{Francis Bach}
%\end{IEEEbiography}

%\begin{biography} ... \end{biography}
%\parpic{\includegraphics[width=1in,clip,keepaspectratio]{example-image-a}}
%\noindent {\bf Joon Ahn} received his B.S. degree in Electrical Engineering from Seoul National University, Seoul, Korea, in 2000. He received his M.S. degree in 2007 and is currently a Ph.D. Candidate in the Department of Electrical Engineering at the University of Southern California. He received the Best Student Paper Award from the Electrical Engineering-Systems Department at the University of Southern California in 2006. His research interests are in the areas of wireless sensor networks, mobile networks, and ad-hoc networks with emphasis on mathematical modeling and performance analysis.

%
 \begin{wrapfigure}{l}{25mm} 
    \includegraphics[width=1in,height=1.25in,clip,keepaspectratio]{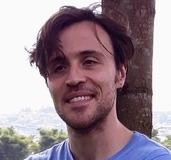}
  \end{wrapfigure}\par
  \textbf{Robert M. Gower} is a visiting researcher at Facebook AI Research (2020 New York) and  joined T\'{e}l\'{e}com Paris as an Assistant Professor in 2017. He is interested in designing and analyzing new stochastic algorithms for solving big data problems in machine learning and scientific computing. A mathematician by training, his academic studies started with a Bachelors and a Masters degree in applied mathematics at the state University of Campinas (Brazil), where he designed the current state-of-the-art algorithms for automatically calculating high order derivatives using back-propagation.\par \vspace{0.2cm}
%%His PhD in stochastic numerics at the University of Edinburgh earned him the 2nd place of the 2017 Leslie Fox prize in numerical analysis.\par
%%   After which in 2016 he was granted the Fondation Sciences Mathématiques de Paris postdoctral Laureate fund to continue his work as a postdoc in ENS. 
%  
   \begin{wrapfigure}{l}{25mm} 
    \includegraphics[width=1in,height=1.25in,clip,keepaspectratio]{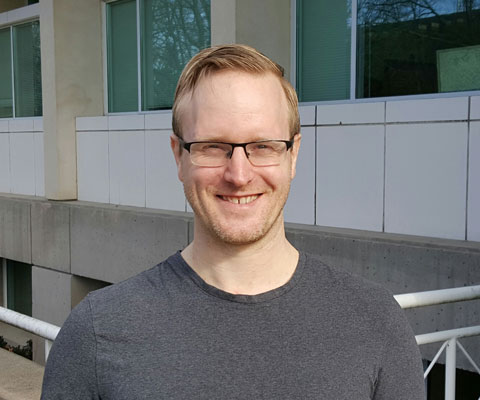}
  \end{wrapfigure}\par
  \textbf{Mark Schmidt} is an associate professor in the Department of Computer Science at the University of British Columbia. His research focuses on machine learning and numerical optimization. He is a Canada Research Chair, Alfred P. Sloan Fellow, CIFAR Canada AI Chair with the Alberta Machine Intelligence Institute (Amii), and was awarded the SIAM/MOS Lagrange Prize in Continuous Optimization with Nicolas Le Roux and Francis Bach. \par \vspace{0.2cm}
   \begin{wrapfigure}{l}{25mm} 
    \includegraphics[width=1in,height=1.25in,clip,keepaspectratio]{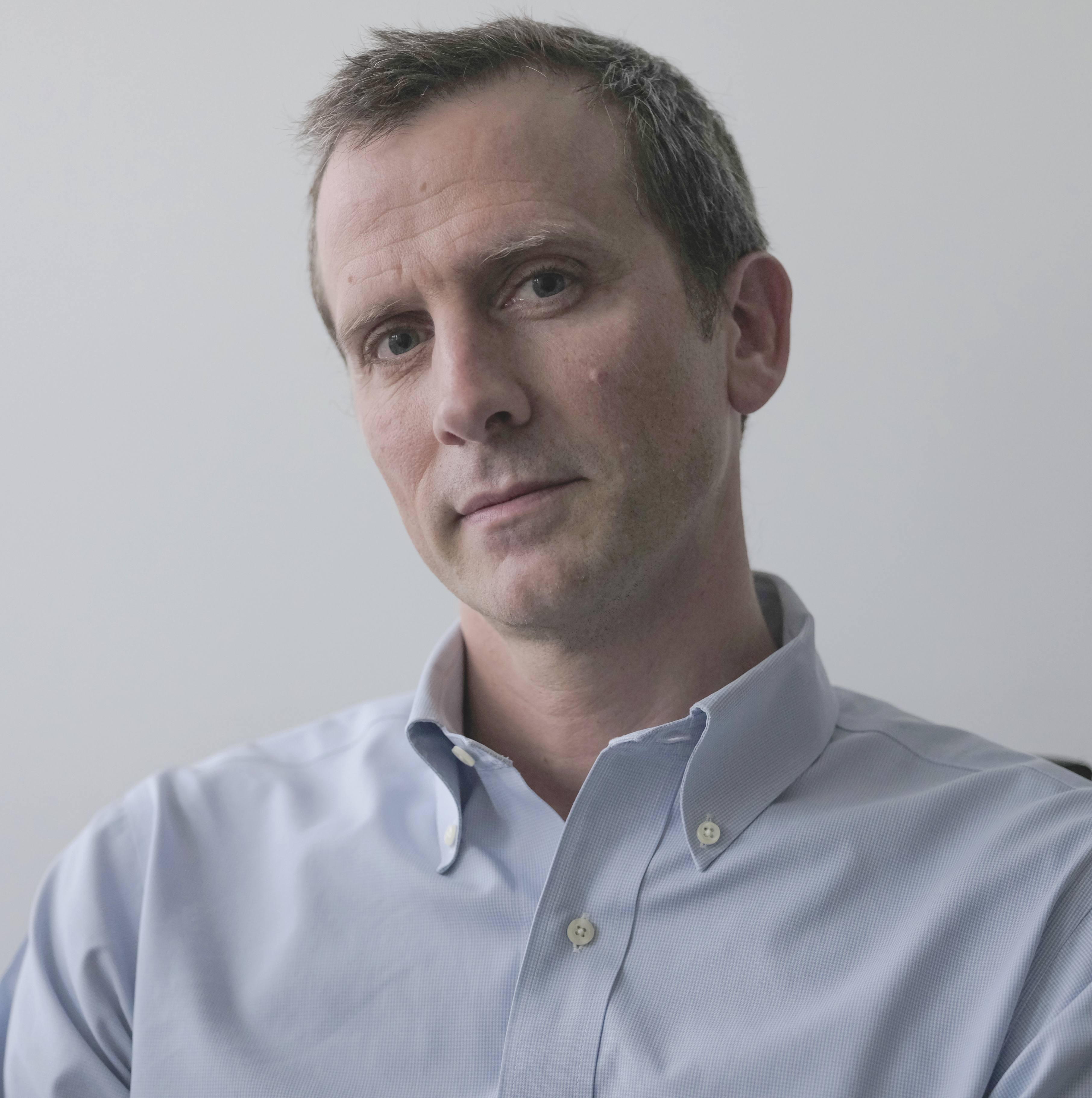}
  \end{wrapfigure}\par
  \textbf{Francis Bach} is a researcher at Inria, leading since 2011 the machine learning team which is part of the Computer Science department at \'{E}cole Normale Sup\'{e}rieure. He graduated from \'{E}cole Polytechnique in 1997 and completed his Ph.D. in Computer Science at U.C. Berkeley in 2005, working with Professor Michael Jordan. He spent two years in the Mathematical Morphology group at \'{E}cole des Mines de Paris, then he joined the computer vision project-team at Inria/\'{E}cole Normale Sup\'{e}rieure from 2007 to 2010. Francis Bach is primarily interested in machine learning, and especially in sparse methods, kernel-based learning, large-scale optimization, computer vision and signal processing. He obtained in 2009 a Starting Grant and in 2016 a Consolidator Grant from the European Research Council, and received the Inria young researcher prize in 2012, the ICML test-of-time award in 2014, as well as the Lagrange prize in continuous optimization in 2018, and the Jean-Jacques Moreau prize in 2019. In 2015, he was program co-chair of the International Conference in Machine learning (ICML), and general chair in 2018. He is now co-editor-in-chief of the Journal of Machine Learning Research.\par \vspace{0.2cm}
   \begin{wrapfigure}{l}{25mm} 
    \includegraphics[width=1in,height=1.25in,clip,keepaspectratio]{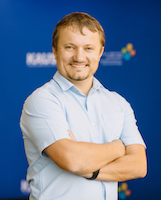}
  \end{wrapfigure}\par
  \textbf{Peter Richt\'arik} is a  Professor of Computer Science and Mathematics at KAUST. He is an EPSRC Fellow in Mathematical Sciences, Fellow of the Alan Turing Institute, and is affiliated with the Visual Computing Center and the Extreme Computing Research Center at KAUST. Prof. Richt\'arik received his PhD from Cornell University in 2007, and then worked as a Postdoctoral Fellow in Louvain, Belgium, before joining Edinburgh in 2009, and KAUST in 2017. His research interests lie at the intersection of mathematics, computer science, machine learning, optimization, numerical linear algebra, and  high performance computing.
 Through his recent work on randomized decomposition algorithms (such as randomized coordinate descent methods, stochastic gradient descent methods and their numerous extensions, improvements and variants), he has contributed to the foundations of the emerging field of big data optimization, randomized numerical linear algebra, and stochastic methods for empirical risk minimization. Several of his papers attracted international awards \revised{to his collaborators}, including the SIAM SIGEST Best Paper Award, the IMA Leslie Fox Prize (2nd prize, three times), and the INFORMS Computing Society Best Student Paper Award (sole runner up).\par

\end{document}